\pdfoutput=1

\documentclass[11pt]{article}

\usepackage[final]{acl}

\usepackage{times}
\usepackage{latexsym}

\usepackage[T1]{fontenc}

\usepackage[utf8]{inputenc}

\usepackage{microtype}

\usepackage{inconsolata}

\usepackage{graphicx}
\usepackage{amsmath}
\usepackage{amssymb}
\usepackage{mathtools}
\usepackage{amsthm}
\usepackage{thmtools,thm-restate}

\theoremstyle{plain}
\newtheorem{theorem}{Theorem}[section]
\newtheorem{lemma}[theorem]{Lemma}
\newtheorem{lemmafake}{Lemma 5.2}

\usepackage{array, multirow, bigdelim, makecell, booktabs} 
\usepackage{subcaption}
\usepackage{float}
\usepackage{thm-restate}

\title{Backward Lens: Projecting Language Model Gradients\\ into the Vocabulary Space}

\author{Shahar Katz$^{1}$ ~~~~~ Yonatan Belinkov$^{1}$ ~~~~~ Mor Geva$^{2}$ ~~~~~ Lior Wolf$^{2}$\\
$^1$Faculty of Computer Science, Technion -- Israel Institute of Technology\\
$^2$Blavatnik School of Computer Science, Tel Aviv University\\
\small{\texttt{\{shachar.katz@cs,belinkov@\}technion.ac.il}},
\small{\texttt{\{morgeva@tauex,wolf@cs\}.tau.ac.il}},
}

\begin{document}
\maketitle
\begin{abstract}
Understanding how Transformer-based Language Models (LMs) learn and recall information is a key goal of the deep learning community.
Recent interpretability methods project weights and hidden states obtained from the forward pass to the models' vocabularies, helping to uncover how information flows within LMs.
In this work, we extend this methodology to LMs’ backward pass and gradients. 
We first prove that a gradient matrix can be {cast} as a low-rank linear combination of its forward and backward passes' inputs. 
We then develop methods to project these gradients {into vocabulary items} and explore the mechanics of how new information is stored in the LMs' neurons.

\end{abstract}

\section{Introduction}
\label{Introduction}

Deep learning models consist of layers, which are parameterized by matrices that are trained using a method known as backpropagation. This process involves the creation of gradient matrices that are used to update the models' layers. Backpropagation has been playing a major role in interpreting deep learning models and multiple lines of study aggregate the gradients to provide explainability \cite{simonyan2014deep, sanyal2021discretized, chefer2022optimizing, sarti-etal-2023-inseq, miglani2023using}.

Recent interpretability works have introduced methods to project the weights and intermediate activations of Transformer-based LMs \cite{vaswani2017attention} into the vocabulary space. The seminal ``Logit Lens'' method \cite{Nostalgebraist2020} has paved the way to explaining LMs' behavior 
during inference \cite{geva-etal-2022-lm, Dar2022AnalyzingTI, ram-etal-2023-token}, including directly interpreting individual neurons \cite{geva2021transformer, katz2023visit}. Our work is the first, as far as we can ascertain, to project LM \emph{gradients} to the vocabulary space.
Furthermore, modern LMs contain thousands of neurons in each layer, while certain features are likely distributed across multiple neurons \cite{elhage2022toy,cunningham2023sparse}. These issues are handled in our examination of the gradient matrices by performing a decomposition of provably low-rank matrices.


\begin{figure}[t]
    \centering
    \includegraphics[width=\columnwidth]{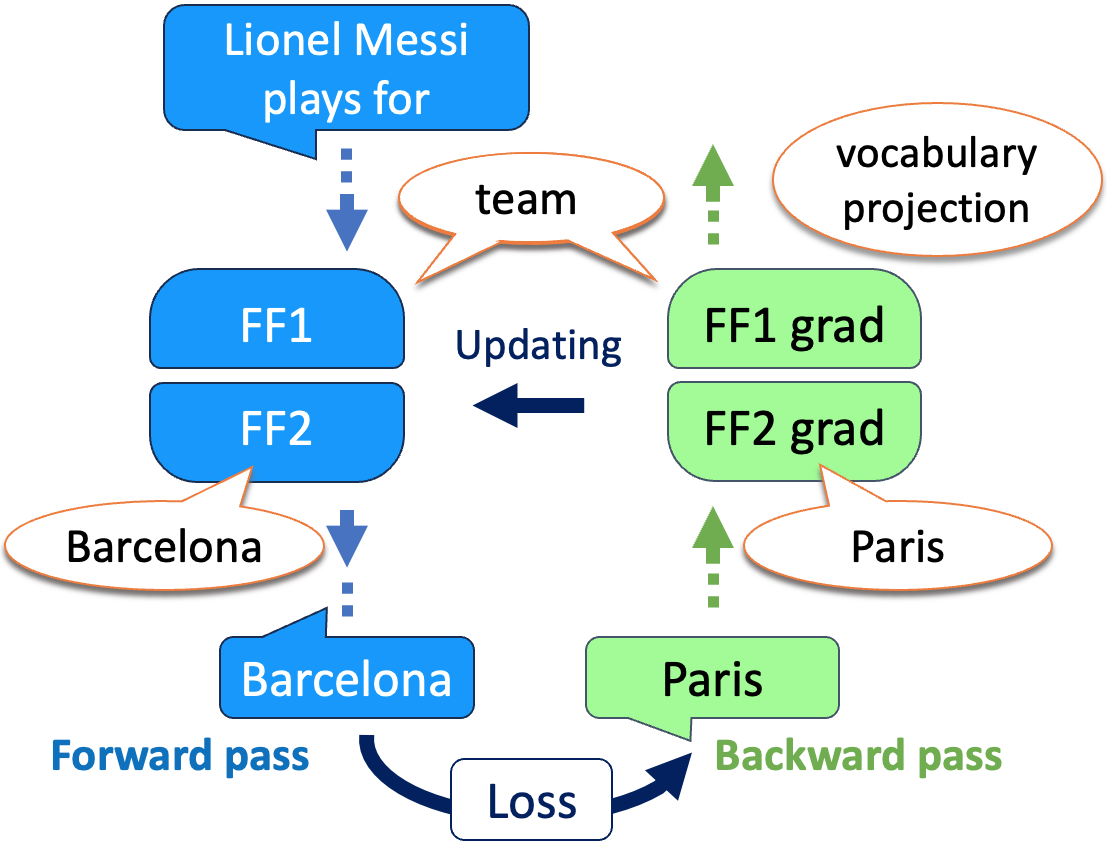}
    \caption{
    An illustration depicting the tokens promoted by a single LM's MLP layer and its gradient during the forward and backward pass when editing the model to answer ``Paris'' for the prompt ``Lionel Messi plays for''. The gradients (in green) of the first MLP matrix, $FF_1$, attempt to imprint into the model's weight (in blue) the information that $FF_1$ encountered during the forward pass.
    Utilizing a vocabulary projection method, we reveal that this information represents the token ``team''.
    The gradients of the second MLP matrix, $FF_2$, aim to shift the information encoded within $FF_2$ towards the embedding of the new target.}
    \label{fig:backward lens}
\end{figure}

Despite the popularity of LMs, our understanding of their behavior remains incomplete \cite{bender2021dangers, dwivedi2023so}, particularly regarding how LMs acquire new knowledge during training and the mechanisms by which they store and recall it \cite{dai2022knowledge, geva2021transformer, geva-etal-2023-dissecting, meng2022mass}. 
We identify a mechanism we refer to as ``imprint and shift'', which captures how information is stored in the feed-forward (MLP) module of the transformer layer. This module has two fully connected layers, $FF_1$ and $FF_2$. The ``imprint'' refers to the first layer, to or from which the learning process adds or subtracts copies of the intermediate inputs encountered during the forward pass. The ``shift'' refers to the second matrix, where the weights are shifted by the embedding of the target token, see \autoref{fig:backward lens}.

In summary, our contributions are: 
(i) Analyzing the rank of gradients.  
(ii) Interpreting gradients by inspecting relatively small spanning sets. 
(iii) Investigating the embedding of these sets by projecting them into tokens, including (iv) examining the Vector-Jacobian Product (VJP) obtained during the backward pass (the equivalent of the hidden states of the forward pass). 
Furthermore, (iv) we reveal a two-phase mechanism by which models learn to store knowledge in their MLP layers,
and (v) leverage it to explore a novel editing method based solely on a single forward pass.

\section{Related Work}
Developing methods to explain LMs is central to the interpretability community \cite{belinkov-glass-2019-analysis, srivastava2023beyond}. Initially inspired by interpretability efforts in vision models \cite{Wojciech2017, zhang2018visual, indolia2018conceptual, olah2020naturally}, LMs have also benefited from the ability to operate in the language domain. This includes leveraging projections of vectors into readable concepts \cite{Nostalgebraist2020, simhi2023} or clustering them into the idioms they promote \cite{cunningham2023sparse, tamkin2023codebook, tigges2023linear, bricken2023monosemanticity}. 

Reverse engineering the gradient's role in shifting model behavior has been a primary method to comprehend the mechanics of deep learning models. Recent work
\cite{ilharco2022editing, gueta2023knowledge, tian2023scan} demonstrates that clustering the weights that models learn during training or fine-tuning reveals patterns that connect the tasks and their training data. In general, existing works examine gradients by observing full matrices, whereas our approach involves interpreting them using the backward pass's VJP. 
Specifically, approaches employing Saliency Maps \cite{simonyan2014deep} explore the relationship between a gradient matrix and parts of its corresponding forward pass input, where we have observed that gradients are spans, linear combinations, of those inputs.

Our experiment regarding LM editing adds to the line of work that utilizes interpretability for knowledge editing. The closest idea to our implementation was introduced by \citet{dai2022knowledge}, who identified activated neurons for specific idioms in encoder LMs and altered them by injecting the embedded target. We show that gradients work in a very similar way. Other state-of-the-art model editing methods include \citet{mitchell2021fast} and  \citet{meng2022locating, meng2022mass}.

Our work analyzes the gradient's rank. It was previously known that gradients are low-rank \cite{mitchell2021fast}, but the utilization of this characteristic for interpretability study or predicting the rank of an edited prompt have remained unexplored.
Optimizers and adaptors, such as LoRA \citet{hu2022lora}, were created to constrain the rank of gradients, making fine-tuning faster. Our work, in contrast, shows that the gradients are already low-rank and utilizes this phenomenon.

\section{Background} 
We first provide background on transformers, focusing on the components that play a role in our analysis and omitting other concepts, such as Layer Norms and positional embedding, which are explored in full by \citet{vaswani2017attention, Radford2019LanguageMA}. We then provide the necessary background on the backward pass; a more comprehensive description is given by \citet{clark2017computing, bishop2006pattern}. Finally, we discuss the building blocks of the Logit Lens method.

\subsection{Transformer LMs}
The Generative Pre-trained Transformer (GPT), is an auto-regressive family of architectures containing multiple transformer blocks. Given a prompt sequence of $n$ tokens, GPT predicts a single token. The architecture maintains an embedding dimension $d$ throughout all layers. First, the input tokens are embedded using an embedding matrix $E$ into the input vectors $X\in \mathbb{R}^{n\times d}$. This is mirrored at the final stage, in which a decoding matrix $D$ projects the output of the last transformer block into a score for each token within the vocabulary.

Each transformer block comprises an attention layer (Attn) and a Multi-Layer Perceptron (MLP) layer, interconnected by a residual stream. 
The attention mechanism transfer vectors (information) from each of the preceding inputs to the current forward pass. In our study, we do not delve into this module and refer the reader to \citet{radfordimproving} for more details.

The MLP layer (also known as FFN, Feed-Forward Network) consists of two fully connected matrices $FF_1$, $FF_2^T\in \mathbb{R}^{d\times d_m}$, with an activation function $f$ between them: $\text{MLP}(X)=f(X FF_1) FF_2$.

Hence, the calculation that the $l$-th transformer block performs on its input hidden state, $X^l$, is given by     $X^{l+1} = X^l + \text{Attn}(X^l) + \text{MLP}(\text{Attn}(X^l)+X^l)$.

\subsection{Backpropagation}
\label{Backpropagation}
Backpropagation \cite{rumelhart1986learning, le1988theoretical} is an application of the chain rule to compute derivatives and update weights in the optimization of deep learning network-based models. 
The process begins with the model executing a forward pass, generating a prediction $\hat{y}$, which is subsequently compared to a desired target by quantifying the disparity through a loss score $L$. Following this, a backward pass is initiated, iterating through the model's layers and computing the layers’ gradients in the reverse order of the forward pass.

For a given layer of the model that during the forward pass computed $z=xW$, where $x \in \mathbb{R}^{m}, z\in \mathbb{R}^{n}$ are its intermediate input and output, we compute its gradient matrix using the chain rule: 
\begin{equation}
\label{formula:chain rule}
\frac{\partial L}{\partial W}=\frac{\partial z}{\partial W}\frac{\partial L}{\partial z} {\in\mathbb{R}^{n\times m}}
\end{equation} 
We can directly compute $\frac{\partial z}{\partial W}=\frac{\partial x W}{\partial W}=x^\top$.
The other derivative $\delta=\frac{\partial L}{\partial z}\in \mathbb{R}^n$ is known as the Vector-Jacobian Product (VJP) of $z$. It can be thought of as the hidden state of the backward pass and is the error factor that later layers project back.

In LMs, the output of the model is an unnormalized vector, $\hat{y}\in\mathbb{R}^{|\text{vocabulary}|}$, representing a score for each of the model's tokens. We denote the target token by an index $t\in [|\text{vocabulary}|]$. Typically the 
Negative Log-Likelihood (NLL) loss is used:
\begin{align}
\label{eq: softamx and nll forward pass}
\hat{p}=&\text{Softmax}(\hat{y})\in\mathbb{R}^{|\text{vocabulary}|}
\\
L=&\text{NLL}(\hat{p},t)=-\log{(\hat{p}[t])}\in\mathbb{R}
\end{align}
where $\hat{p}$ represents the normalized probabilities of $\hat{y}$ and $[t]$ is its $t$-th value (the target token's probability).
For the last layer's output $z=\hat{y}$, calculating its $\delta$ (VJP) can be done directly by ($k \in [|\text{vocabulary}|]$):
\begin{align}
\label{formula: NLL direct VJP}
    \delta[k] = \begin{cases}
    \hat{p}[k]-1\leq0 \ \ \ \ \ \ \ \  & \text{if } k = t \\
    \hat{p}[k]\geq0  & \text{otherwise}
\end{cases}
\end{align}

For an earlier layer in the model $l$, we cannot compute the VJP of its output $z^l$ directly (here $^{l}$ indicates the layer's index). 
Since we iterate the model in a reverse order, we can assume we already computed the VJP of layer $l+1$.
If the layers are sequential, the output of layer $l$ is the input of $l+1$, therefrom $z^l=x^{l+1}$. Utilizing the backward step, we can compute:
\begin{equation}
\label{formula: x VJP by backward step}
\delta^l=\frac{\partial L}{\partial z^l}=\frac{\partial L}{\partial x^{l+1}}=\delta^{l+1} (W^{l+1})^\top
\end{equation}
To summarize, in deep learning models, the gradient of a loss function $L$ with respect to a given layer $W$, is the outer product of the layer's forward pass input, $x$, and its output $z$'s VJP, $\delta$:
\begin{equation}
\label{formula:x^t*delta}
\frac{\partial L}{\partial W}=\frac{\partial z}{\partial W}\frac{\partial L}{\partial z}=x^\top \cdot \delta \in{\mathbb{R}^{n\times m}}
\end{equation}

\subsection{Vocabulary Projection Methods}
\label{Logit Lens and Projecting Methods}

\citet{Nostalgebraist2020} discovered that we can transform hidden states from LMs forward passes into vocabulary probabilities, thereby reflecting their intermediate predictions. Termed as \textbf{Logit Lens (LL)}, this method projects a vector $x$ in the size of the embedding space $d$ by applying it with the LM's decoding, the process that transforms the last transformer block's output into a prediction: 
\begin{equation}
\label{eq: LL}
    \text{LL}(x)=\text{Softmax}(ln_f(x)D)\in\mathbb{R}^{|\text{vocabulary}|}
\end{equation}
where $ln_f$ is the model’s last Layer Norm before the decoding matrix $D$.

The projection captures the gradual building of LMs output \cite{millidge2022singular, haviv2022understanding}, and projections from later layers are more interpretable than earlier ones. Efforts such as \citet{din2023jump, belrose2023eliciting} try to solve this gap by incorporating learned transformations into LL. However, to emphasize our main discoveries, we have not included such enhancements, which primarily aim to shortcut the models’ computations and require dedicated training procedures.

An artificial neuron performs a weighted sum of its inputs, and appears as a column or a row of the model's matrices 
taken along a direction that has a dimensionality $d$.  
Static neurons can also be projected into tokens using LL: \citet{geva2021transformer}, \citet{geva-etal-2022-transformer} observe that neurons of the first MLP matrix $FF_1$ determine the extent to which each neuron in $FF_2$ contributes to the intermediate prediction. \citet{Dar2022AnalyzingTI}, \citet{geva-etal-2023-dissecting} employ the same approach to investigate the attention matrices. \citet{elhage2021mathematical}, \citet{katz2023visit} demonstrate how these neurons can elucidate model behavior, \citet{wang-etal-2023-label}, \citet{millidge2022singular}, \citet{todd2023function} use it to explore circuits and in-context learning. 

Despite the growing interest in this approach, we are only aware of works that have applied it to the static weights of models or the hidden states of the forward pass. In contrast, our work is focused on the backward pass of LMs.

\section{Backward Lens}
\label{sec:backlens}
In this section, we detail the methods we developed to analyze gradients based on our understanding of how each gradient matrix is formed.

\subsection{Gradients as Low-Rank Matrices} 
\label{LMs' Gradints' Ranks}
\citet{hu2022lora} and \citet{mitchell2021fast} have observed the low-rank of MLP layers' gradients with a single input. 
However, they did not explain this phenomenon in the context of a matrix with a sequence of inputs, nor did they predict this rank. The following Lemma does both.

\begin{lemma}
\label{lemma: ranks}
    Given a sequence of inputs of length $n$, a parametric matrix $W$ and a loss function $L$, the gradient $\frac{\partial L}{\partial W}$  produced by a backward pass is a matrix with a rank of $n$ or lower.
\end{lemma}
\begin{proof}
According to \autoref{formula:x^t*delta}, the gradient of a matrix is $\frac{\partial L}{\partial W}=x^\top \cdot \delta$. Assuming $x,\delta$ are non-zero vectors, the rank of the gradient matrix is 1, given its interpretation as a span of a single column vector $x$, or equivalently, as a span of a single row $\delta$. In the case when $x$ or $\delta$ is a zero vector, the rank of the gradient matrix is 0. 

In LMs, an input prompt comprises a sequence of $n$ tokens, each of which introduces an intermediate input ($x_i$) at every layer. In this case, the gradient matrix is the sum of each $x_i, \delta_i$'s product:  
\begin{equation}
\label{formula: our decomposition}
\frac{\partial L}{\partial W}=\sum_{i=1}^{n}\frac{\partial z_i}{\partial W}\frac{\partial L}{\partial z_i}=\sum_{i=1}^{n} x_i^\top \cdot \delta_i
\end{equation}
The maximum rank of the summed gradient matrix is $n$ given each $x_i$, or $\delta_i$,  is linearly independent, since we sum $n$ distinct rank-1 matrices.
Reasons for this rank to be lower than $n$ are the existence of linear dependencies between $x_i$ or between $\delta_i$, with 0 being the minimum possible rank.
\end{proof}

Of particular interest is the case of the last layer of the transformer. In this case, the rank of the gradient is one, see \autoref{The Rank of The Last Layer}.

\subsection{Applying Logit Lens to Gradient Matrices}
\label{Applying Logic Lens to Gradient Matrices}
\label{Interpreting the MLP’s Gradients via Decomposition}
In our analysis we focus on the MLP layers, due to recent interest in identifying and editing the knowledge stored in these layers \cite{geva-etal-2022-transformer, geva2021transformer, dai2022knowledge, mitchell2021fast, meng2022locating}.
Consider the MLP modules $FF_1$ and $FF_2$. The first maps from $\mathbb{R}^d$ to $\mathbb{R}^{d_m}$, which is typically, for many transformers,  $4d$. The second maps from the latter dimension to the former. In both cases, the gradient matrix has one dimension of $d_m$ and one of $d$. Exploring all these dimensions is prohibitive, see \autoref{Why Decomposed Gradient Analysis Makes Sense}. 
However,  \autoref{formula: our decomposition} reveals that every gradient matrix is a sum of $n$ outer products $x_i^\top \cdot \delta_i$. 
This view allows us to examine every gradient matrix as a sum of $n$ pairs of vectors. Our analysis, therefore, would focus on only $n<<d_m$ vectors in $\mathbb{R}^d$. 

Every matrix formed by $x_i^\top \cdot \delta_i$ can be interpreted in two ways simultaneously: (1) as a span (linear combination) of $x_i$ and (2) as a span of $\delta_i$. \autoref{fig:span explained} illustrate the two viewpoints.
We utilize this duality and examine gradients as the linear combinations of $n$ vectors: $x_i$ or $\delta_i$.

\begin{figure}[t]
    \centering
    \includegraphics[width=\columnwidth]{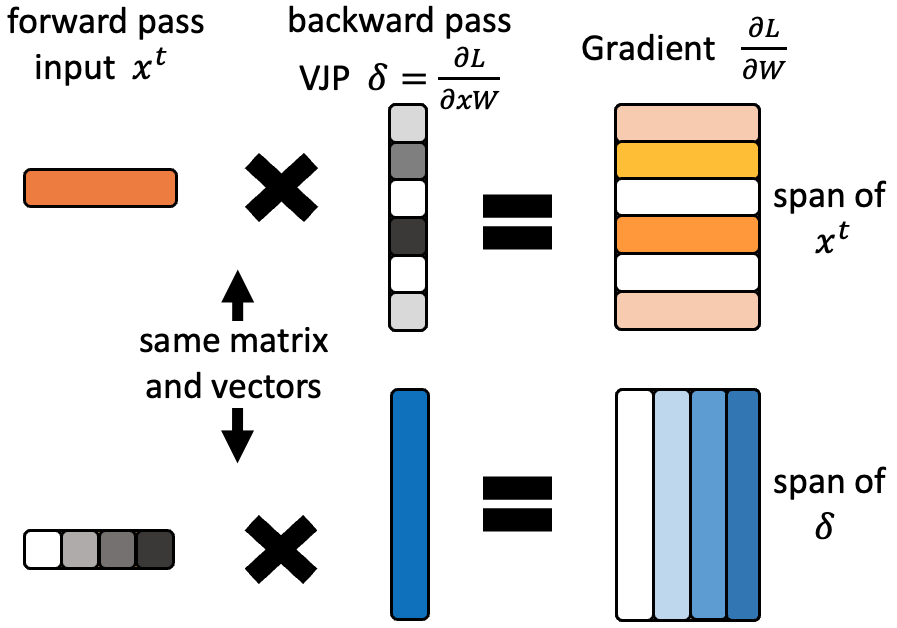}
    \caption{The calculation of gradient matrix by the outer product of $x^\top \cdot \delta$. Each row consists of the same values, but above we describe the matrix as a span of $\delta$, while below as a span of $x^\top$. The displayed vectors are presented transposed to emphasize the spanning effect.}
    \label{fig:span explained}
\end{figure}

\paragraph{The gradients of $FF_1$}\label{The gradients of $FF_1$}
$\delta_i$ is not of size $d$, but $x_i$ are $d$-sized vectors and were already explored using LL \cite{geva-etal-2022-transformer, Dar2022AnalyzingTI}.
Therefore, for $FF_1$ we chose to observe the gradient matrix as a span of $x_i$.
Explicitly, we refer to $x_i$ as $FF_1$'s \textbf{spanning set} since the $j$-th neuron of the gradient matrix is equal to a linear combination of $x_i$:
\begin{equation}
\label{eq:RVD FF1}
    \frac{\partial L}{\partial FF_1}[j]=\sum_{i=1}^n {x_i^\top} \cdot \delta_i [j]\,,
\end{equation}
where $\delta_i [j]$ is the $j$-th element of the vector $\delta_i$.

\paragraph{The gradients of $FF_2$}\label{The gradients of $FF_2$}
The sizes of $FF_2$'s $x_i, \delta_i$ are switched from those of $FF_1$, hence
we chose $\delta_i$ as $FF_2$'s gradient spanning set. Thus, its gradient's $j$-th neuron is viewed as a combination of $\delta_i$:
\begin{equation}
\label{eq:RVD FF2}
\frac{\partial L}{\partial FF_2}[j]=\sum_{i=1}^n \delta_i\cdot x_i [j]\,,
\end{equation}
where $x_i [j]$ is the $j$-th element of the vector $x_i$. In \autoref{Understanding the Backward Pass Process} we provide a theoretical explanation of why the choice of the VJP $\delta_i$ is not only a technical one, due to dimensionality considerations. 

\begin{figure*}[t]
\centering
\includegraphics[width=0.95\textwidth]{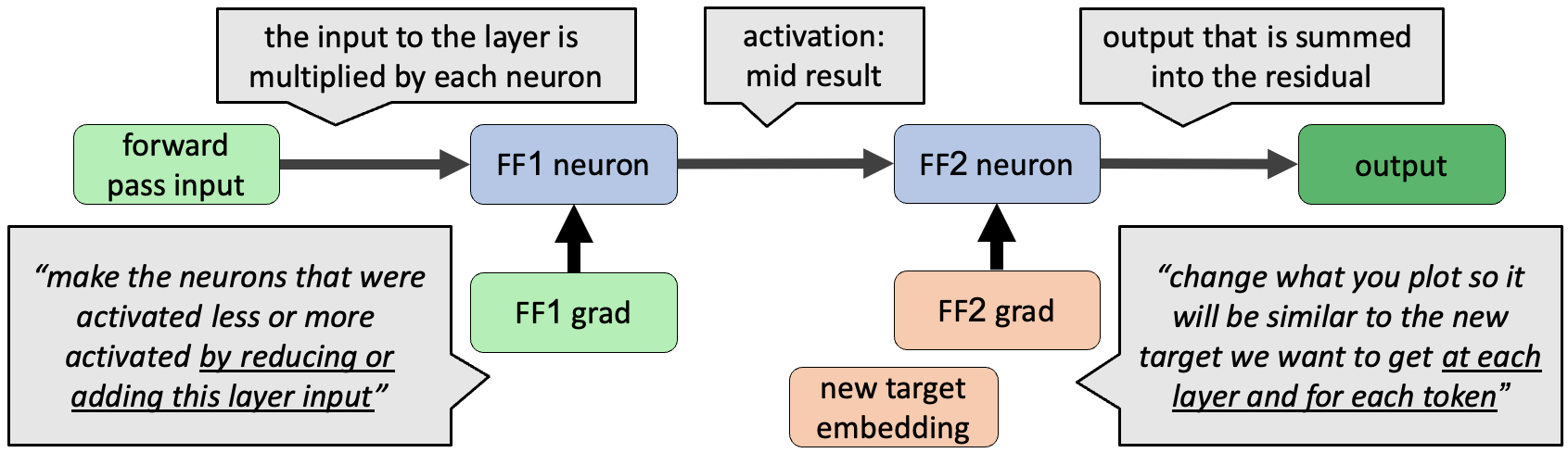}
\caption{The \textbf{Imprint and Shift} mechanism of backpropagation.``grad'' represent a single neuron from a gradient matrix. The color of $FF_1$ grad is the same as the forward-pass input, while $FF_2$ is the same as the new target embedding, suggesting that they are similar to each other.}
\label{fig:trap and shift}
\end{figure*}

\section{Understanding the Backward Pass}
\label{sec:understandbackpass}
\label{Understanding the Backward Pass Process}

The VJPs, $\delta_i$, are the hidden states of the backward pass and the vectors that constitute the gradient matrices (\autoref{Backpropagation}). In this section, we try to shed light on what information is encoded in the VJP. Additionally, we aim to explain how \autoref{eq:RVD FF1} and \autoref{eq:RVD FF2} cause the model to change its internal knowledge.
For simplicity, we ignore Dropouts and Layer Norms.

\subsection{The VJPs of the Top Layer}
\label{The VJPs of the Top Layer}
In this section, we analyze the initial VJPs that are created during the editing of a single prompt, as we describe in \autoref{Backpropagation}.
The last matrix of an LM, which is the final model parameter used before calculating the loss score, is the decoding matrix $D\in\mathbb{R}^{d\times|\text{vocabulary}|}$.
During the forward pass, this matrix calculates the output vectors $x_iD=\hat{y}_i$.
When editing a prompt, we only use the final prediction of the last prompt's token $x_nD=\hat{y}_n$ to calculate the loss score.

We calculated $D$'s VJP of the last token $\delta_n$.
According to \autoref{formula: x VJP by backward step}, the backward pass' VJP to the layer that preceded $D$ is:
\begin{equation}
    \frac{\partial L}{\partial x_n}=\frac{\partial L}{\partial \hat{y}}D^\top=\delta_n D^\top\in\mathbb{R}^d
\end{equation}
This result can be simplified as a weighted sum of $D$'s columns:
\begin{equation}
\delta_n D^\top=\sum_{k=1}^{|\text{vocabulary}|}\delta_n[k]D^\top[k]\,,
\end{equation}
where $D^\top[k]\in \mathbb{R}^{d}$ is the $k$-th neuron of $D$ and also {the embedding of the model's $k$-th token.} 
From the equation, $\delta_n[k]\in\mathbb{R}$ controls the magnitude by which we add the embedding of the $k$-th token into $\frac{\partial L}{\partial x_n}$. 
Pluggin  \autoref{formula: NLL direct VJP} and using the notation $\hat{p}=\text{Softmax}(\hat{y}_n), t$ from \autoref{eq: softamx and nll forward pass}:
\begin{lemma}
\label{observation: initial backward vjp}
The VJP $ \delta_n D^\top $ passed at the beginning of a backward pass is a vector in $\mathbb{R}^d$ that is a sum of weighted token embeddings.
It is dominated by the embedding of the target token, $D^\top[t]$, multiplied by a negative coefficient $\delta_n[t]=\hat{p}[t]-1$. The embedding of all other tokens $k\neq t$ are scaled by a positive coefficient $\hat{p}[k] D^\top[k]$.
\end{lemma}

If we ignore Dropouts and Layer Norms, the VJP $\delta_nD^\top$ is the initial vector to be passed in the backward pass. In particular, this is the {only VJP to span the last MLP layer's gradient. The use of residual streams implies that while the backward pass iterates the model in a reverse order, this vector skips to previous layers, hence it will be part of the span of all the MLPs' gradients.

The LL of $\delta_n$ is provided as $\text{Softmax}(ln_f(\delta_n D^\top)$. Except for $ln_f$,
this behaves similarly to $\text{Softmax}(\delta_n D^\top)$. As the lemma shows, $\delta_n D^\top[t]$ is negative, while $\delta_n D^\top[k]$ is positive for all tokens $k\neq t$. By assuming homogeneity of the embedding vectors and independence of the coefficients, we can expect the target embedding to have the lowest probability in the softmax. In practice, since related tokens have more similar embeddings and similar entries in $y_n$, this effect is expected to be even more pronounced.

\subsection{Storing Knowledge in LMs}
\label{Trap and Shift}
In \autoref{Applying Logic Lens to Gradient Matrices} we observed that each neuron in the MLP's gradients is a sum of vectors in $\mathbb{R}^d$ from the forward and backward passes, $x_i$ and $\delta_i$ respectively.
Based on this observation, we aim to understand how LM editing with a single prompt and a single backward pass changes the internal knowledge of a model.
Explicitly, we study the implications of updating a weight matrix with its gradients:
$
    W\leftarrow W +  \eta \frac{\partial L}{\partial W}^\top
$
,where $\eta\in\mathbb{R}$ is a negative learning rate.

\begin{restatable}{lemma}{lemma52}
\label{lemma: imprint and shift}
When updating an MLP layer of an LM using backpropagation and rerunning the layer with the same inputs $x_i$ from the forward pass of the prompt we used for the editing, the following occurs:
(i) The inputs, $x_i$, are added or subtracted from the neurons of $FF_1$,
thereby adjusting how much the activations of each corresponding neuron in $FF_2$ would increase or decrease.
(ii) The VJPs $\delta_i$ are subtracted from the neurons of $FF_2$, amplifying in $FF_2$'s output the presence of the VJPs after they are multiplied with negative coefficients.
\end{restatable}
\noindent See \autoref{Proof of Lemma 5.2} for the proof.

Since the change in $FF_1$ uses the given inputs $x_i$ to amplify future activation, we term this mechanism ``imprint''. The modification of $FF_2$ is termed as the ``shift'', since it represents a process of altering the output of the layer.
In summary, the ``imprint and shift'' mechanism depicts the MLP's learning process during a single backward pass as having two phases:
Given the layer's original input and the new target, the process imprints a similar input through the update of $FF_1$ and subsequently shifts the output of $FF_2$ towards the new target. \autoref{fig:trap and shift} illustrates this process.

\textbf{LL ranking} refers to the index assigned to the vocabulary's tokens when ordered by the probability scores generated by LL (\autoref{eq: LL}). Updating $FF_1$ involves adding or subtracting $x_i$ from weights, focusing on the most probable tokens from the LL ranking. Conversely, for $FF_2$, updating entails subtracting $\delta_i$, effectively adding $-1 \cdot \delta_i$. 
This subtraction reverses the LL rankings, turning previously least probable tokens into most probable ones. Thus, when utilizing LL with $FF_2$'s $\delta_i$, attention should be given to the least probable tokens from the projections.

\section{Experiments}
\label{ref:exp}
We conduct a series of experiments to support the results of Sec.~\ref{sec:backlens} and~\ref{sec:understandbackpass}, as well as to briefly demonstrate their application to LM analysis.

We employ GPT2 \cite{Radford2019LanguageMA} and Llama2-7B \cite{touvron2023llama} in our experiments. 
We randomly sampled 100 prompts and their corresponding  editing targets from the CounterFact dataset \cite{meng2022locating}. For each model and prompt, we conducted a single backpropagation using SGD and without scaling optimizers, such as Adam~\citep{kingma2014adam}, and no batching.

\begin{figure}[t]
\includegraphics[width=\columnwidth]{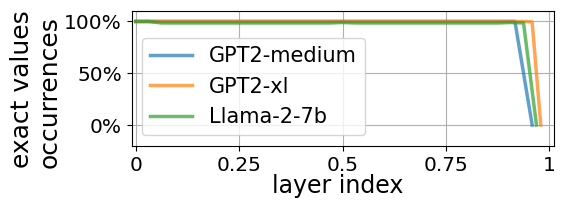}
    \caption{The percentage of occurrences where the rank of $FF_1$'s gradient equals the length of the prompt used for editing.
    To show different models in the same plot, we normalize the layer indices.
    Except for the last layer, all layers and models exhibit the above equality more than $98.5\%$ of the time.}
    \label{fig:FF1 ranks vs prompt len}
\end{figure}

\paragraph{The rank of the gradients} To examine Lemma~\ref{lemma: ranks}, we measure the rank of each layer gradient matrix. 
As depicted in \autoref{fig:FF1 ranks vs prompt len}, for every prompt with the length of $n$ tokens, the model's gradient matrices are almost always exactly rank $n$. The only exceptions are the last MLP layers, which have a rank of 1, as predicted in \autoref{sec:backlens}.  
Although unnoticeable from the figure, once in a few dozen examples, there is a drop of one or two in the rank of the gradients, indicating linear dependency in $x_i$ or $\delta_i$, see \autoref{LMs' Gradints' Ranks}. This is not a result of a repeated token, since the positional encoding would still lead to a different $x_i$.

\begin{figure}[t]
\centering
\includegraphics[width=\columnwidth]{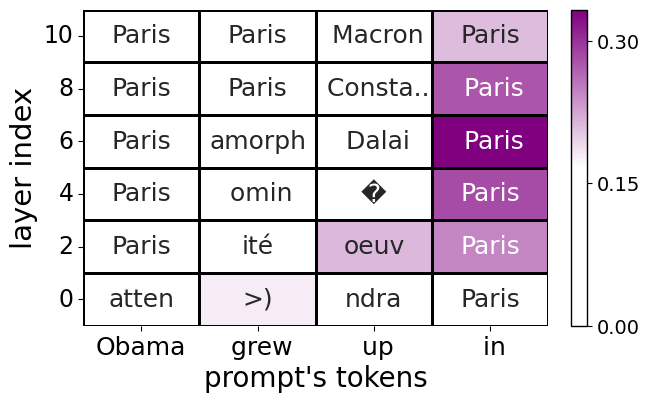}
\caption{The gradient of GPT2-small $FF_2$ when editing the model to answer ``Paris'' for the prompt ``Obama grew up in''. 
Each cell shows the Logit Lens projection of the gradient's VJP ($\delta_i$) for a token input and a layer.
Non-English characters are replaced with a question mark, and long tokens are truncated with ``..''. 
According to \autoref{Trap and Shift}, instead of showing the most probable token in each cell, we display the least probable one. The color indicates the norm of the VJP, with white cells indicating that almost no editing is done in practice.}
\label{fig:obama paris}
\end{figure}

\paragraph{Logit Lens of Gradients} 
Next, we present examples of our gradients' interpretation through LL in \autoref{fig:obama paris} and in \autoref{Additional Table Examples}.
In each cell of these plots, the LL projections of the chosen spanning set ($FF_1$'s $x_i$ and $FF_2$'s $\delta_i$) are presented for a specific layer and a token from the 
prompt that was used for the editing.

Prior studies that projected the forward pass examine the LL projections of hidden states, highlighting the gradual change in the projected tokens between layers \cite{Nostalgebraist2020, haviv2022understanding}. Similarly, \autoref{fig:obama paris} presents a gradual change in the backward pass' VJP.  
Across most layers, LL reveals that the gradients represent the embedding of ``Paris''. 
 Other projections have semantics that are related to ``Paris'' such as ``Macron'', the family name of the President of France. The norm of the VJP is indicted by color, and, in the top layers, the only meaningful updates are for the token ``Paris''. Some of the edits in the lower layers are harder to explain, similarly to the situation in those layers for the vanilla LL of the forward pass.

\paragraph{Impact of Different Segments of the Prompt}\label{Editing Mass}\label{What Tokens The Gradients Represent}

We observe that while all the prompt's tokens contribute to the gradient construction (\autoref{formula: our decomposition}), the majority of these contributions are done by VJPs, $\delta_i$, with a close-to-zero norm. 
Furthermore, upon examining the LL of every individual neuron from the gradient matrix (Appendix \ref{Why Decomposed Gradient Analysis Makes Sense}), we found that all the projected tokens are correlated with only 1-2 vectors we can identify from the spanning sets presented in \autoref{Applying Logic Lens to Gradient Matrices}.

\begin{figure}[t]
    \centering
    \includegraphics[width=\columnwidth]{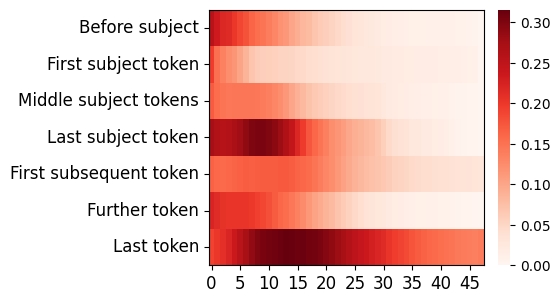}
    \caption{The norm of GPT2-xl's $FF_2$'s VJPs ($\delta_i$) as a function of the layers' index and segments of the edited prompts. 
    White color represents close-to-zero updates with almost no effect on the model's weights.}
    \label{fig:tracing1}
\end{figure}

To discern the relative importance of tokens and layers in the gradient reconstruction, we divide each prompt's tokens into segments and plot their $\delta_i$ mean norm. This experiment is done with GPT2-xl, due to its extensive use in prior work on interpretability research.

Results are depicted for $FF_2$ in \autoref{fig:tracing1}, see Appendix~\ref{Editing Mass of Every Layer} for $FF_1$. Evidently, predominant updates occur in two main areas: (1) by the subject's tokens in the initial layers, and (2) by the last prompt's token around the second quarter of the layers.
The majority of other tokens exhibit a norm close to zero throughout the layers, indicating that they have almost no effect on the updating. 
We hypothesize that the changes to the last subject token may involve editing the information transferred by the subject's token through attention, as demonstrated by \citet{geva-etal-2023-dissecting}.

A complementary view is provided by considering, for the LL rank of each VJP $\delta_i$  (labeled by the segment of token $i$ of the input) the rank of the target token.
\autoref{fig:target token in FF2} illustrates that the VJP of the last token from the edited prompts, $\delta_n$, consistently ranks the target token among the least probable ones. The VJPs of other tokens from the edited prompt, $\delta_i$, exhibit comparable behavior, generally ranking the target token as improbable.

\begin{figure}[t]
\includegraphics[width=1\linewidth]{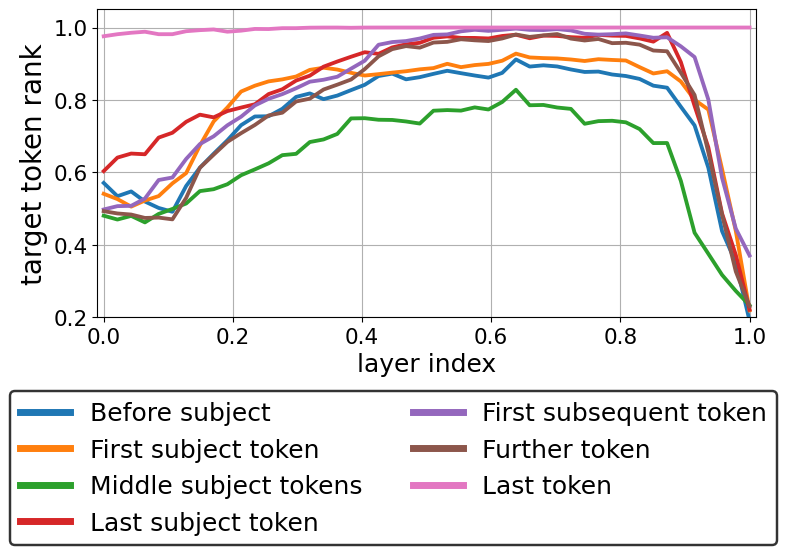}
\caption{
The Logit Lens rank of the target token for GPT2-xl $FF_2$'s VJPs, $\delta_i$.
Most gradients tend to rank the target token as one of the least probable tokens, with the last token consistently ranking it as such.
We found that the degradation of the first and last layers can be attributed to the proximity of certain $\delta_i$ norms to zero.
In the initial layers, we posit that Logit Lens is less effective, thereby resulting in lower readability for the earlier layers.}
\label{fig:target token in FF2}
\end{figure}

The result reveals that along the first and last layers, some of the $\delta_i$ show degradation in their ranking of the target token, which we attribute to their low norms as reflected by \autoref{fig:tracing1}. We demonstrate in \autoref{Normalize Logit Lens} that normalizing $\delta_i$ before LL magnifies the presence of the target token.
Specifically, the drop at the model's last layer is  due to the fact that apart from the last prompt's token, all the others have a zero vector $\delta_i$ at that layer (\autoref{The Rank of The Last Layer}).

Please note that the degradation of this rank in the first few layers might be related to the gap in LL interpretability for the earlier layers discussed in \autoref{Logit Lens and Projecting Methods}. In Appendix~\ref{The Ranks of $FF_1$ and the Models' Original Answer} we provide a similar analysis for $FF_1$'s gradients.

\section{Application: Editing Based on the ``Shift'' Mechanism}
\label{Editing Directly with Target Embedding}
Prior work \cite{mitchell2021fast, meng2022locating, meng2022mass} introduced editing methods that change only $FF_2$'s matrices.
In \autoref{Trap and Shift} we identify the ``shift'' mechanism of editing $FF_2$. In \autoref{What Tokens The Gradients Represent} we observe that the dominant components in constructing the gradients are derived from the outer product of the last token's input and VJP, $x_n^\top \cdot \delta_n$, and that $\delta_n$ contains the embedding of the target token.

We hypothesize that we can edit LMs' internal knowledge by updating only a single $FF_2$ matrix with a single forward pass and an approximation of the VJP, thereby eliminating the need for the backward pass.
Based on \autoref{observation: initial backward vjp}, the embedding of the editing target is annotated by $D^\top[t]$, where $D$ is the decoding matrix and $t$ is the index of the target token.
Our experimental method works as follows: (i) We choose an MLP layer we wish to edit, $\hat{FF_2}$ (predefined as a hyper-parameter according to \autoref{appendix: editing}). (ii) We run a single forward pass with the prompt whose output we want to edit. (iii) During the forward pass, we collect the last token input for the layer we want to edit, $x_n$. (iv) We collect the embedding of the target token $D^\top[t]$ and (v) update the MLP matrix by $\hat{FF_2}\leftarrow \hat{FF_2} + \eta\cdot x_n^\top \cdot D^\top[t]$, where $\eta$ is the learning rate.
We name this method ``forward pass shifting''.

We examined our method on 1000 samples from CounterFact (\autoref{tab: Manual grad main paper table}; the full results and additional implementation details are presented in \autoref{Manual Gradients Editing: appendix}), and found that for single editing our approach is on par with the state-of-the-art methods MEND \cite{mitchell2021fast}, ROME \cite{meng2022locating} and MEMIT \cite{meng2022mass}, in editing a given prompt, but it falls short in comparison to ROME in generalization (editing paraphrases) and specificity (see Appendix). 
However, our method has much lower runtime complexity and does not employ a multi-step (iterative) execution. 
Overall,  our results suggest we might be able to find ``shortcuts'' in the implementation of fine-tuning by injecting tokens directly into LMs' layers.

\begin{table}[t]
\begin{center}
\begin{sc}
\begin{tabular}{@{}l@{~}c@{~}c@{~}c@{}}
\toprule
Method & Eff$\uparrow$ & Par$\uparrow$ & n-gram$\uparrow$ \\
\midrule
Original model &  0.4 & 0.4 & 626.94 \\
Finetuning (MLP 0) & 96.4  & 7.46 &  618.81 \\
Finetuning (MLP 35) & 100.0 & 46.1 &  618.50 \\
\midrule
MEND   & 71.4  & 17.6 &  623.94 \\
ROME     & 99.4  &  71.9 & 622.78 \\
MEMIT     & 79.4  & 40.7 & 627.18 \\
\textbf{Forward pass shift} & 99.4 & 41.6  &  622.45 \\
\bottomrule
\end{tabular}
\end{sc}
\end{center}
\caption{GPT2-xl single editing results for CounterFact.
Efficacy (EFF) represents the editing success rate (accuracy). 
Paraphrase (PAR) denotes the accuracy of predicting the new target for phrases derived from the edited prompt.
N-gram measures generation fluency using weighted bi- and tri-gram entropies.}
\label{tab: Manual grad main paper table}
\end{table}

\section{Conclusions}
 Other LL-type interpretability contributions shed light on LMs through the forward pass. Here, we show that gradients can be projected into the vocabulary space and utilize the low-rank nature of the gradient matrices to explore the backward pass in an interpretable way. As we show, the gradients are best captured by a spanning set that contains either the input to each layer, or its VJP.  
These two components, which are accessible during the forward and backward passes, are used to store information in the MLP layers, using a mechanism we call ``imprint and shift''. 
We provided experimental results to substantiate the results of our analysis, including an editing method that only requires a single forward pass, but is on par with the SOTA knowledge editing methods.

\section{Limitations}
Our use of LL in projecting gradients has limitations when it comes to explaining the gradients of earlier layers. At this point, it remains unclear whether gradients operate in the same embedding space across all layers or if another transformation is required for projecting earlier layers.
This question is currently being explored for the forward pass (see \autoref{Logit Lens and Projecting Methods}), suggesting additional learned transformations to the first layers. 
Given the lack of a wide consensus on this additional transformation, we have opted to employ only the original LL projection in our analysis.
Furthermore, some recent contributions against LL argue that this method is more correlated with LMs' behaviors, rather than causally explaining them. Our work shows that at least in the later layers of LMs, token embeddings are directly placed into the weights of the LM, making LL projections well-justified.

Recently, alternative approaches have been proposed to explain LMs by intervening in the forward pass \cite{meng2022locating}. When combined with token projection methods, this approach holds promise in providing insights into the ``thinking’’ process of LMs \cite{ghandeharioun2024patchscope}.

Our work ignores the additional scaling that is introduced by optimizers other than Stochastic Gradient Descent, such as  Adam \cite{kingma2014adam}. 
While the backward pass's VJPs remain unaffected when such optimizers are employed, they do alter the rank and weights of each gradient matrix, due to the additional scaling.

Our approach to explaining how knowledge is stored in LMs is grounded in single editing with a constant embedding. While our approach elucidates how models store various information, fine-tuning is typically conducted on multiple prompts and involves multiple steps (iterations). Additionally, training a model from scratch includes the training of its embeddings. 

Our experimental approach to editing LMs with ``forward pass shift'' is presented as a case study rather than as a suggested alternative to existing methods. The results in \autoref{Editing Directly with Target Embedding}, ~\autoref{Manual Gradients Editing: appendix} might obfuscate ``editing’’ and ``output shifting’’, since only plotting the desired answers does not fully encapsulate the effect of the edit on similar prompts, which is a challenge faced by most editing benchmarks and datasets.

Our focus on the MLP layers excludes the attention layers. This decision is influenced by the growing consensus that MLPs are where LMs predominantly store information \cite{dai2022knowledge, meng2022locating}. We acknowledge the possibility that attention layers may also store information and that editing MLPs and attention simultaneously could have different effects on the model from those detailed in \autoref{Trap and Shift}. 

Our theoretical analysis disregards certain components of LMs, such as Dropouts, Layer Norms, positional embedding and bias vectors. 
We acknowledge that these components may have distinct effects on the interpretation of the backward pass, but without these simplifications the derivations are laden with additional terms.

Lastly, our work was conducted on Decoder LMs with sequential architecture. It is important to note that other types of LMs might exhibit different behaviors in terms of their gradients.

\smallskip
\section{Ethics and Impact Statement}
This paper presents work whose goal is to advance the field of Machine Learning. There are many potential societal consequences of our work, none which we feel must be specifically highlighted here.
However, future research could use the methods we developed to edit LMs. We hope such cases would be for developing better and safer models, rather than promoting harmful content.

\smallskip
\smallskip
\section*{Acknowledgements}
 This work was supported by the ISRAEL SCIENCE FOUNDATION (grant No.\ 448/20), an Open Philanthropy alignment grant, and an Azrieli Foundation Early Career Faculty Fellowship.

\bibliography{anthology_tmp, custom}

\appendix

\clearpage

\section{The Rank of The Last Layer} \label{The Rank of The Last Layer}
\label{app:B}
In \autoref{LMs' Gradints' Ranks}, we delve into the observation that each gradient matrix has a rank equal to the length of the edited prompt (annotated by $n$), except for the last layer's ones. In this section, we explain why the last layer's MLP matrices are always rank-1.

The backward pass, applied to the final loss score ($L$, \autoref{Backpropagation}), generates a computational graph that is reversed in direction compared to the forward pass \autoref{Backpropagation}. It begins with the loss score and the matrices of the last layer, proceeding in reverse order until reaching the matrices of the first layer. This computational graph encapsulates every hidden state and intermediate result that contributed to the final prediction, which is the output of the last layer for the last token in the prompt.

One might initially assume that, since the last prediction was formed by the input of the last token, only its hidden states would be involved in this computational graph. However, due to the attention mechanism, hidden states from previous forward passes can be recalled and utilized in subsequent forward passes, contributing to all the tokens that follow them in the prompt.
The last hidden state to be recalled using the attention modules is called at the model's last layer’s attention module, which precedes an MLP module in sequential architectures, such as GPT2 and Llama-7B.
Hence, in every layer, MLP inputs in the reverse computational graph comprise all individual intermediate inputs $x_i$ from the forward pass of each token in the prompt. However, at the last layer, the only input included is the one belonging to the last prompt token $x_n$. For this reason, also only the VJP of the last token, $\delta_n$, is included in the reconstruction of the gradients of the last MLP layer, while the $\delta_i$ for all the other tokens from the prompt are not included (or more correctly, they are equal to the zero vectors).

When constructing the gradients using $x_i$ and $\delta_i$, the rank of each layer is equal up to the number of $x_i, \delta_i$ involved in the computational graph (assuming linear independence \autoref{LMs' Gradints' Ranks}). This implies that all layer matrices are formed by $\frac{\partial L}{\partial W}=\sum_{i=1}^n x_i^\top \cdot \delta_i$ except for the last layer, which is constructed with $\frac{\partial L}{\partial W}=x_n^\top \cdot \delta_n$, which is rank-1.

In our study, especially in our figures and tables, we decided to include all the vectors of the last layer, including those from tokens which are not the prompt's last, which are thus equal to zero vectors. 
This approach is also the reason for the observed changes in the behavior of the gradients in some figures. 
For example, in \autoref{fig:target token in FF2} we can see that all the graphs (except for the last token's) converge to the same value at the last layer. The reason for this is that they are all equal to the zero vector. 
In \autoref{fig:table_medium_messi} we see the LL projection of the VJPs from the model's last layer, which are equal to projecting the zero vector.

\section{Why Decomposed Gradient Analysis Makes Sense}
\label{Why Decomposed Gradient Analysis Makes Sense}
In \autoref{Logit Lens and Projecting Methods}, \autoref{Applying Logic Lens to Gradient Matrices} we establish our interpretation of gradients via spanning sets. This approach is based on the understanding that each neuron in the gradient matrix is formed by the linear combination of $x_i$ (the forward pass's intermediate inputs) or $\delta_i$ (VJPs, the backward pass's hidden states).
In this section, we aim to illustrate, through a singular example, why analyzing a gradient matrix through its spanning set is more informative and simpler compared to attempting to analyze the full gradient matrix.

We use GPT2-medium (24 layers and 330M parameters) for our examination. We examine the MLP gradients using the prompt ``Lionel Messi plays for’’, to which the model responds with ``Barcelona’’. We edit the model with a single backward pass to respond ``Paris’’. In the case of this model, each MLP matrix comprises 4096 neurons. Consequently, to apply the Logit Lens (LL) projection to a particular gradient matrix, the process needs to be applied 4096 times.

\begin{table*}[h!]
\vskip 0.15in
\begin{center}
\begin{sc}
\begin{tabular}{cccc}
\toprule
Group & norm & LL top & LL bottom \\
\midrule
\multirow{3}{4em}{Biggest by norm} & 1.212 & \verb|Paris|, \verb| Paris|, \verb| Marse| & \verb|ufact|, \verb| Logged|, \verb|otomy| \\
& 0.352 & \verb|Paris|, \verb| Paris|, \verb| Marse| & \verb|ufact|, \verb|Spanish|, \verb| Gerr| \\
& 0.297 & \verb| Paris|, \verb|Paris|, \verb| Copenhagen| & \verb| ceremonial|, \verb|cade|, \verb|uana| \\
\hline
\multirow{3}{4em}{Medium by norm} & 0.033 & \verb|ufact|, \verb| Gerr|, \verb| sheriff| & \verb|Paris|, \verb| Paris|, \verb|ienne| \\
& 0.033 & \verb| Logged|, \verb| turtle|, \verb| ceremon| & \verb|Paris|, \verb| Paris|, \verb|France| \\
& 0.033 & \verb|ufact|, \verb| sec|, \verb| recess| & \verb|Paris|, \verb| Paris|, \verb|qus| \\
\hline
\multirow{3}{4em}{Smallest by norm} & 0.001 & \verb|,|, \verb| the|, \verb| and| & \verb|VIDIA|, \verb|advertisement|, \verb|Dialogue| \\
& 0.001 & \verb|,|, \verb| the|, \verb| and| & \verb|VIDIA|, \verb|advertisement|, \verb|Magikarp| \\
& 0.001 & \verb|,|, \verb| the|, \verb| and| & \verb|VIDIA|, \verb|advertisement|, \verb|Companies| \\
\hline
\multirow{2}{4em}{zero vector} & 0 & \verb|,|, \verb| the|, \verb| and| & \verb|VIDIA|, \verb|advertisement|, \verb|Companies| \\
& & & \\
\bottomrule
\end{tabular}
\end{sc}
\caption{Sample of gradient's neurons projection via Logit Lens (LL) from GPT2-medium, $FF_2$ matrix, layer 14.
\textbf{LL TOP} stands for the most probable tokens via LL, while \textbf{BOTTOM} are the most improbable ones.
In this example, we edit the prompt ``Lionel Messi plays for'' with the editing target ``Paris''. 
In the projected tokens we notice the predominance of ``Paris'', and also that gradient's neurons with a relatively low norm project the same tokens as the zero-vector.
}
\label{tab:messi FF2 samples}
\end{center}
\vskip -0.1in
\end{table*}

\begin{table*}[h!]
\vskip 0.15in
\begin{center}
\begin{sc}
\begin{tabular}{cccc}
\toprule
prompt token & norm & LL top & LL bottom \\
\midrule
\verb|L| & 0.029 & \verb| Rutherford|, \verb| Apost|, \verb| PROG| & \verb|Paris|, \verb| Paris|, \verb|ienne| \\
\verb|ion| & 0.035 & \verb|ramid|, \verb|ngth|, \verb| livest| & \verb|France|, \verb|ée|, \verb|É| \\
\verb|el| & 0.067 & \verb| unlaw|, \verb|owship|, \verb|arantine| & \verb| Libyan|, \verb| Libya|, \verb|France| \\
\verb| Messi| & 0.165 & \verb|ğ|, \verb| relic|, \verb| ejected| & \verb|vu|, \verb|igmat|, \verb|tain| \\
\verb| plays| & 0.026 & \verb| surv|, \verb| POV|, \verb| NTS| & \verb|Paris|, \verb| hotels|, \verb|Merit| \\
\verb| for| & 0.165 & \verb| ceremonial|, \verb|ado|, \verb|| & \verb|Paris|, \verb| Paris|, \verb| Copenhagen| \\
\bottomrule
\end{tabular}
\end{sc}
\caption{The Logit Lens of the VJPs ($\delta_i$) of GPT2-medium, $FF_2$ matrix, layer 14. Notice the dominance of ``Paris'' (the editing target) in the projected vocabulary and the norm ratio of the vectors.}       
\label{tab:messi FF2 reprasentative vectors}
\end{center}
\vskip -0.1in
\end{table*}

\paragraph{We start by analyzing $FF_2$ from layer 14}. In \autoref{tab:messi FF2 samples} we present samples of gradient neurons' projections by LL. 
In \autoref{The gradients of $FF_2$}, we elaborate on how each neuron is formed by multiplying an interpretable vector by a coefficient ($\delta_i$ and $x_i$), which in turn dictates its norm. 
We group these neurons based on their norms, unveiling shared projections among neurons within the same norm group. To underscore the proximity of certain neurons to the zero vector, we also include the projection of the zero vector.
Expanding \autoref{tab:messi FF2 samples} to every individual neuron in the matrix is impractical, given the challenge of reading a table with 4096 rows. Instead, we use plots \autoref{fig:messi representative vs singles}: the initial plot presents the LL intersection, measuring the extent of overlap between the top 100 most probable  tokens from two vectors, while the second shows the cosine similarity between the vectors.

We repeat the process after sorting the gradient neurons according to their norms \autoref{fig:messi representative vs singles sorted}. 
The gradients with the higher norms, 
are almost identical to the last VJP ($\delta_n$), with alignment extending up to the sign of the vectors.
In \autoref{tab:messi FF2 reprasentative vectors} LL reveals that these VJPs project ``Paris''. 
Gradients with low norms may appear correlated with parts of the spanning set's vectors ($FF_2$'s $\delta_i$), yet they are more correlated to the zero vector, emphasizing that these neurons do not update the model's weights (induce minimal change).

In addressing color shifting, we see around index 500 from the right that this is where the activations change sign from positive to negative. The negative learning rate causes the positive activation to add ``Paris’’ into those neurons, while the negative activation reduces ``Paris’’. In both cases, the process causes the model to add ``paris'' in the same direction (\autoref{Trap and Shift}). 

\begin{figure*}[h!]
\centering
\subcaptionbox{Logit Lens intersection\label{Logit Lens intersection - explain}}{\includegraphics[width=1\linewidth]{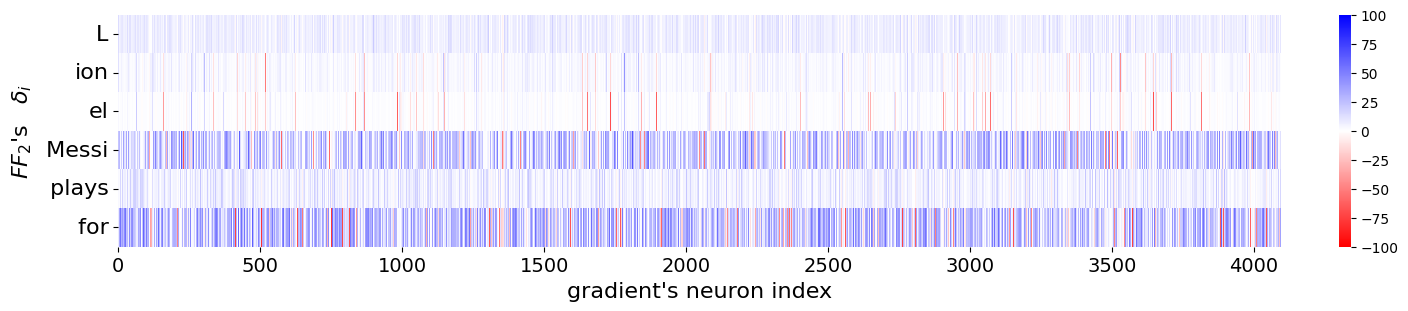}}
\subcaptionbox{Cosine Similarity}{\includegraphics[width=1\linewidth]{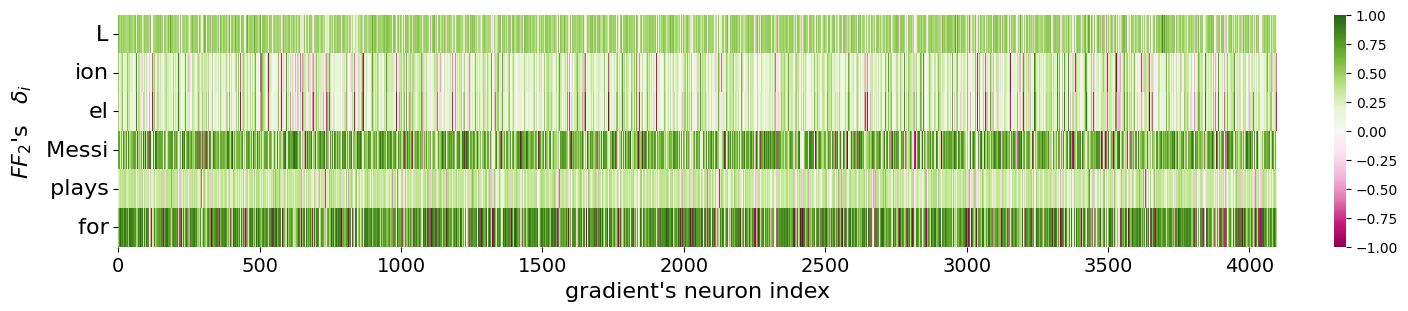}}
\caption{VJPs ($\delta_i$, the spanning set of $FF_2$) and single gradient's neurons comparison for layer 14's $FF_2$.
In \autoref{Logit Lens intersection - explain} a positive score is the amount of shared tokens between the top 100 most probable tokens after projecting each vector.
A negative score is presented if multiplying one of the vectors by $-1$ produces a higher amount of shared tokens (than the score presented with a negative sign).
The reason we apply this $-1$ multiplication is to address cases where two vectors have high correlation up to their sign (similarly to two vectors that overlap each other after one of them is multiplied by $-1$, making their cosine similarity negative).
}
\label{fig:messi representative vs singles}
\end{figure*}

\begin{figure*}[h!]
\centering
\subcaptionbox{Logit Lens intersection}{\includegraphics[width=1\linewidth]{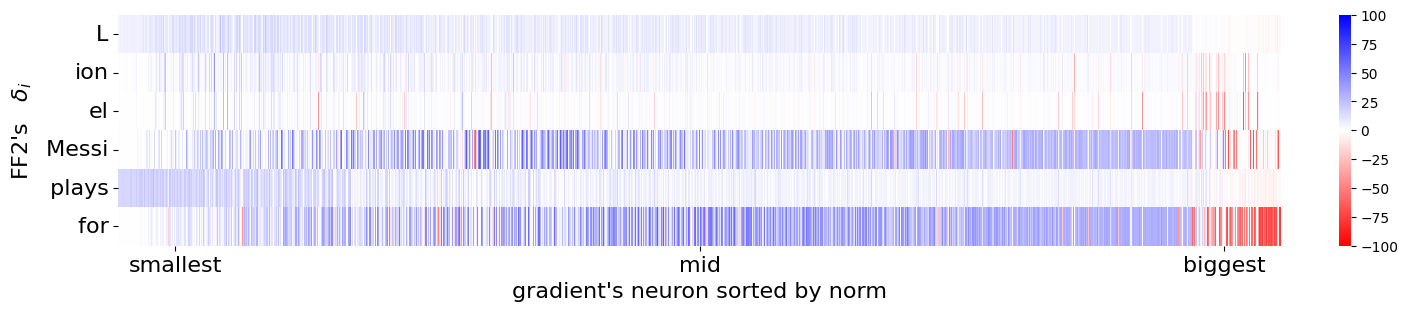}}
\subcaptionbox{Cosine Similarity}{\includegraphics[width=1\linewidth]{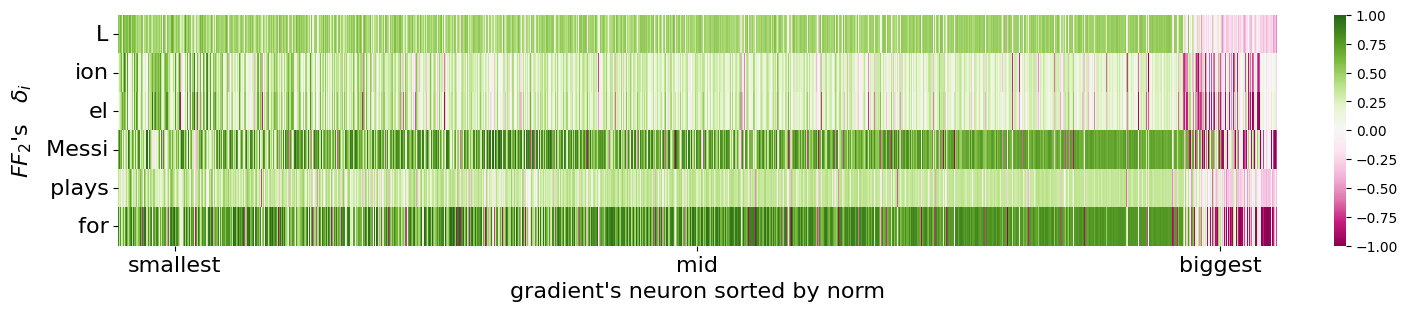}}
\caption{VJPs ($\delta_i$) and single gradient's neurons (sorted by norm) comparison for layer 14's $FF_2$. 
It is noteworthy that high-norm neurons, corresponding to those activated with positive activations during the forward pass, inject the VJP of ``for'' with a flipped sign. Medium-sized norm neurons also exhibit correlation with the representative vector of ``for''. Smallest by-norm neurons show minimal correlation; we refer to their proximity to the zero vector.}
\label{fig:messi representative vs singles sorted}
\end{figure*}


\paragraph{We repeat this type of analysis with $FF_1$}. We remind that our interpretation for this layer’s spanning set is its inputs $x_i$ \autoref{Applying Logic Lens to Gradient Matrices}.
Again, we see the alignment between individually analyzing the neurons of the gradients and the spanning set \autoref{tab:messi FF1 samples}, \autoref{tab:messi FF1 reprasentative vectors}, \autoref{fig:messi representative vs singles sorted FF1}. 


\begin{table*}[h!]
\centering
\vskip 0.15in
\begin{center}
\begin{sc}
\begin{tabular}{cccc}
\toprule
Group & norm & LL top & LL bottom \\
\midrule
\multirow{3}{4em}{Biggest by norm} & 0.438 & \verb| Cruz|, \verb|ization|, \verb|ize| & \verb| mathemat|, \verb| trave|, \verb| nodd| \\
& 0.422 & \verb| Football|, \verb| Jr|, \verb| Team| & \verb|theless|, \verb| challeng|, \verb| neighb| \\
& 0.369 & \verb| psychiat|, \verb| incent|, \verb|theless| & \verb| Jr|, \verb| Junior|, \verb| Sr| \\
\hline
\multirow{3}{4em}{Medium by norm} & 0.02 & \verb| the|, \verb| a|, \verb| hire| & \verb|irez|, \verb|inelli|, \verb|intosh| \\
& 0.02 & \verb| the|, \verb| a|, \verb| one| & \verb|theless|, \verb|Magikarp|, \verb|irez| \\
& 0.02 & \verb| perpend|, \verb| coerc|, \verb| incent| & \verb| Junior|, \verb| Jr|, \verb| Football| \\
\hline
\multirow{3}{4em}{Smallest by norm} & 0.002 & \verb|,|, \verb| the|, \verb| and| & \verb|Magikarp|, \verb|VIDIA|, \verb|advertisement| \\
& 0.002 & \verb|,|, \verb| the|, \verb| and| & \verb|advertisement|, \verb|VIDIA|, \verb|Magikarp| \\
& 0.001 & \verb|,|, \verb| the|, \verb| and| & \verb|VIDIA|, \verb|advertisement|, \verb|Companies| \\
\hline
\multirow{2}{4em}{zero vector} & 0 & \verb|,|, \verb| the|, \verb| and| & \verb|VIDIA|, \verb|advertisement|, \verb|Companies| \\
& & & \\
\bottomrule
\end{tabular}
\end{sc}
\caption{Samples of gradient's neurons projection via Logit Lens (LL) from GPT2-medium, $FF_1$ matrix, layer 14.}
\label{tab:messi FF1 samples}
\end{center}
\vskip -0.1in
\end{table*}

\begin{table*}[h!]
\vskip 0.15in
\begin{center}
\begin{sc}
\begin{tabular}{cccc}
\toprule
prompt token & norm & LL top & LL bottom \\
\midrule
\verb|L| & 9.499 & \verb|,|, \verb||, \verb| the| & \verb|FontSize|, \verb|7601|, \verb|Magikarp| \\
\verb|ion| & 7.808 & \verb|fall|, \verb|fish|, \verb|wood| & \verb| nodd|, \verb| incorpor|, \verb|accompan| \\
\verb|el| & 7.728 & \verb| Cruz|, \verb| McC|, \verb| Esp| & \verb| perpend|, \verb| mathemat|, \verb| shenan| \\
\verb| Messi| & 6.944 & \verb| Jr|, \verb| Junior|, \verb| Sr| & \verb|theless|, \verb| psychiat|, \verb| incent| \\
\verb| plays| & 6.715 & \verb| football|, \verb| golf|, \verb| alongside| & \verb|hedon|, \verb|ilts|, \verb|uries| \\
\verb| for| & 6.596 & \verb| Team|, \verb| team|, \verb| a| & \verb|irez|, \verb| newsp|, \verb|Magikarp| \\
\bottomrule
\end{tabular}
\end{sc}
\caption{The Logit Lens of the inputs ($x_i$) of GPT2-medium, $FF_1$ matrix, layer 14. 
According to our observation \autoref{Trap and Shift}, those are also the embeddings the gradients inject into the model's weights.
} 
\label{tab:messi FF1 reprasentative vectors}
\end{center} 
\vskip -0.1in
\end{table*}

\begin{figure*}[h!]
\centering
\subcaptionbox{Logit Lens intersection}{\includegraphics[width=1\linewidth]{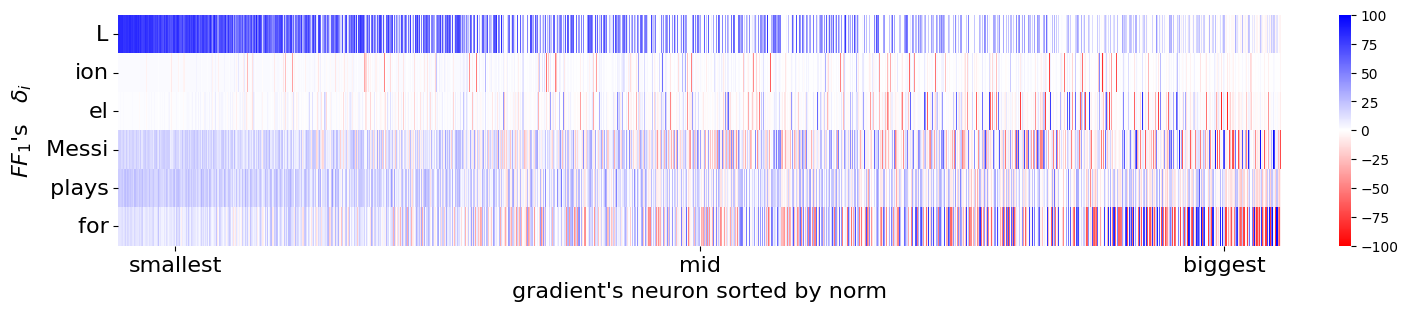}}
\subcaptionbox{Cosine Similarity}{\includegraphics[width=1\linewidth]{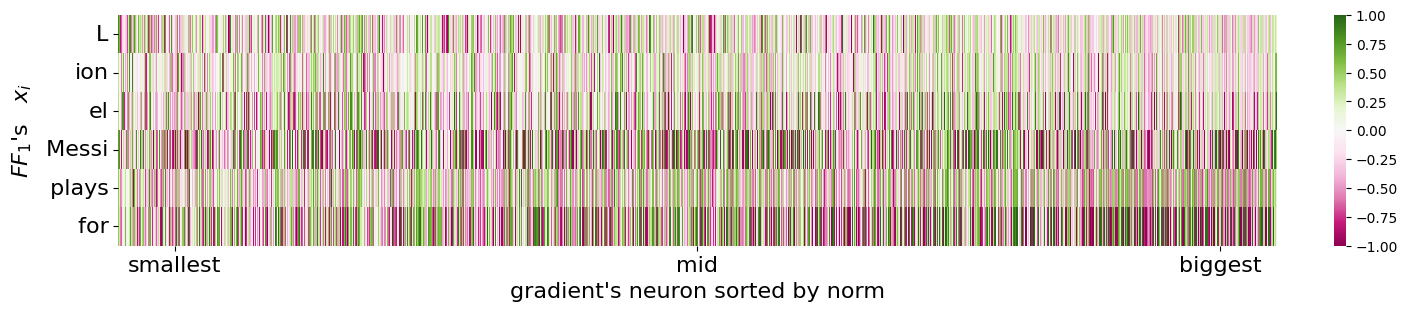}}
\caption{The inputs $x_i$ (which are $FF_1$'s spanning set) and single gradient's neurons (sorted by norm)  comparison for layer 14's $FF_1$.}
\label{fig:messi representative vs singles sorted FF1}
\end{figure*}

\paragraph{In conclusion,}
In this section, we demonstrate the converging results of analyzing individual gradient neurons via LL and the spanning sets interpretation. The aim of this analysis is to emphasize the efficacy of employing these spanning sets to simplify experiments involving a vast number of vectors (neurons) into a much smaller, representative subset. 
The advantage of this simplification is twofold: it conserves computational resources and reduces time expenditure.

\clearpage
\clearpage 

\section{Proof of Lemma 5.2}
\label{Proof of Lemma 5.2}

\begin{lemmafake}
When updating an MLP layer of an LM using backpropagation and rerunning the layer with the same inputs $x_i$ from the forward pass of the prompt we used for the editing, the following occurs:
(i) The inputs, $x_i$, are added to or subtracted from the neurons of $FF_1$,
thereby adjusting how much the activations of each corresponding neuron in $FF_2$ increase or decrease.
(ii) The VJPs $\delta_i$ are subtracted from the neurons of $FF_2$, amplifying in $FF_2$'s output the presence of the VJPs after they are multiplied with negative coefficients.
\end{lemmafake}

\begin{proof}
In \autoref{The gradients of $FF_1$}, we revealed that $FF_1$ weights are updated by injections (adding or subtracting vectors) of its $x_i$.
This update is done according to the coefficients of its $\delta_i$, after multiplying them with the learning rate.
If we rerun the same layer with the same input after the update, we obtain the following output per each neuron $j$:
\begin{equation}
\begin{aligned}
\label{equation: updaing a single neuron}
    & x_i \cdot (\textit{FF}_1[j] + \eta \cdot {x_i^\top} \cdot \delta_i[j]) = \\
    & x_i \cdot \textit{FF}_1[j] + x_i \cdot \eta \cdot {x_i^\top} \cdot \delta_i[j] = \\
    & x_i \cdot \textit{FF}_1[j] + \lVert x_i \rVert_2^2 \cdot \eta \cdot \delta_i[j]\in \mathbb{R}
\end{aligned}
\end{equation}
$x_i \cdot \textit{FF}_1[j]$ is the pre-edit output of this neuron. The second component, $\lVert x_i \rVert_2^2 \cdot \eta \cdot \delta_i [j]$, is derived from the update and can be positive or negative, hence it controls the increment or decrement of this output compared to the pre-edit one.
In models with monotonic (or semi-monotonic) activation functions, such as ReLU (GeLU is positive monotonic only from around $-0.75$), the activation of the corresponding neuron in $FF_2$ will be changed directly by this addition to that output. 

In \autoref{The gradients of $FF_2$}, we show how $FF_2$'s $\delta_i$ form the gradient matrix.
Consider the result of updating only $FF_2$ and rerunning the same layer:
\begin{equation}
\begin{aligned}
\label{equation: updaing a single neuron in FF2}
   & x_i[j]\cdot (\textit{FF}_2[j] + \eta \cdot \delta_i \cdot x_i[j] ) =\\ & x_i[j] \cdot \textit{FF}_2[j] +  \eta \cdot (x_i[j])^2 \cdot \delta_i \in \mathbb{R}^d
\end{aligned}
\end{equation}
$x_i[j] \cdot \textit{FF}_2[j]$ represents the original output of this neuron. Since $(x_i[j])^2$ is always non-negative and $\eta$ is negative, the update is a subtraction of $\delta_i$ from each neuron.

\end{proof}

\vfill\eject

\section{Additional Logit Lens of Gradients Examples}
\label{Additional Table Examples}
The work that presented Logit Lens (LL) \cite{Nostalgebraist2020} established this method by constructing tables of the projected results of different prompts, illustrating the tokens each individual forward pass represents at each layer of the model.
In this section, we present a similar approach, focusing on the backward pass rather than the forward pass.
In the following figures, we provide examples of the LL projections of gradients when they are interpreted as combinations of the intermediate inputs $x_i$ or VJPs $\delta_i$ according to \autoref{Applying Logic Lens to Gradient Matrices}.

In addition to illustrating the information stored in the gradient matrices, the following tables also describe the information moving through LMs during the forward and backward pass.
Our interpretation of $FF_1$ using $x_i$ reflects the gradual buildup in the prediction of the forward pass, as $x_i$ represents the intermediate inputs for each MLP layer.
However, $FF_2$'s $\delta_i$, VJPs, serve as the backpropagation counterpart to the forward pass $x_i$. We can conceptualize them as the input of the MLP when executing the backward pass, or as the error propagated from later layers to earlier ones.
In conclusion, \autoref{fig:table_medium_messi}, ~\ref{fig:table_xl_apple_bmw}, ~\ref{fig:table_xl_pickachu}, ~\ref{fig:table_llama_eddy_cua}, ~\ref{fig:table_llama_messi} depict the information stored within the gradients, simultaneously also comparing the information revealed by LL from the forward pass ($x_i$, known from prior studies) with that from the backward pass ($\delta_i$), which is part of our innovative contribution.

\begin{figure*}[h!]
    \centering
    \includegraphics[width=\textwidth]{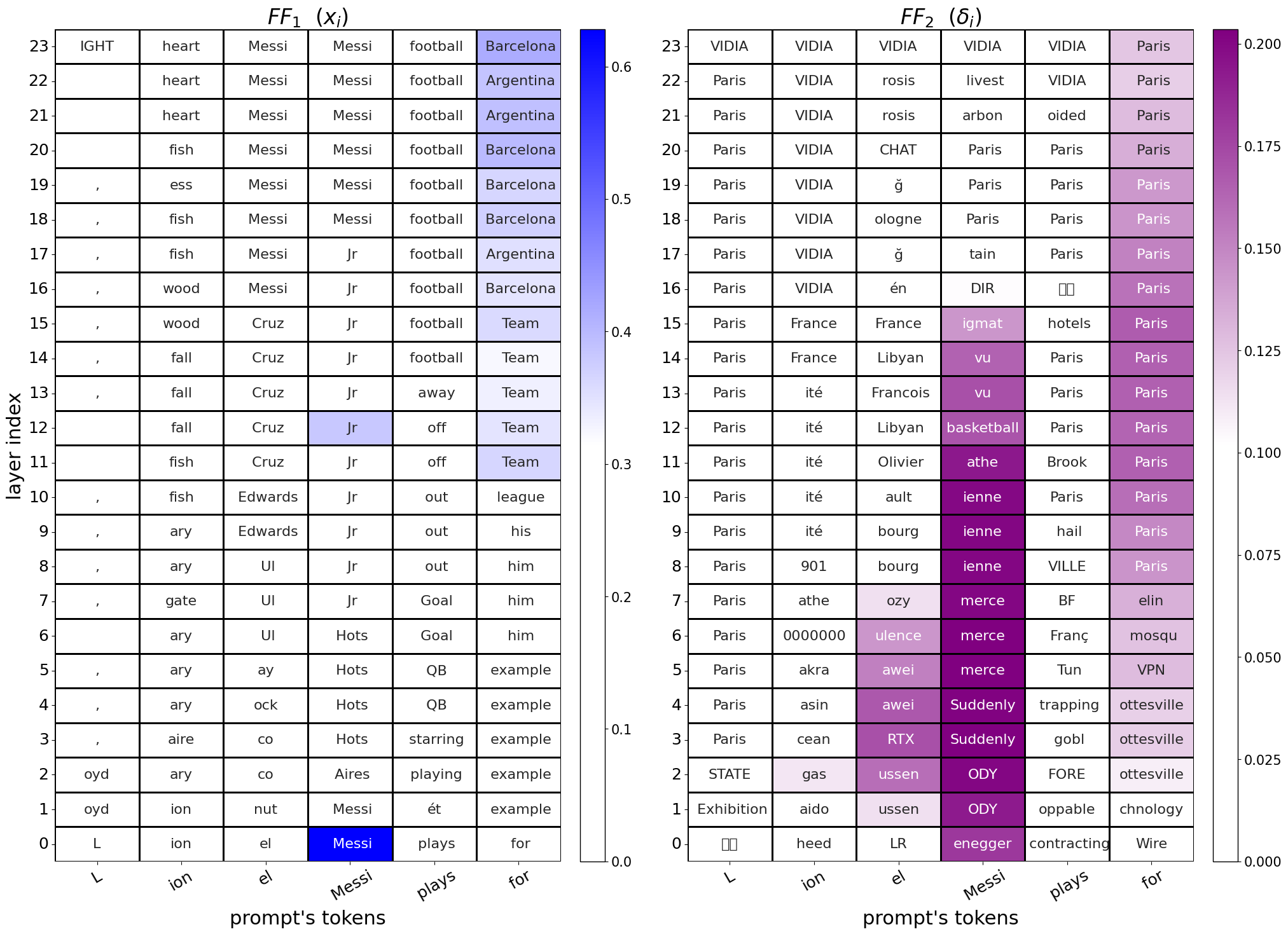}
    \caption{GPT2-medium MLP gradients via our spanning set interpretation. Each cell illustrates the Logit Lens of the gradient's spanning set according to a layer and a token from the edited prompt.
    According to \autoref{Trap and Shift} observation, for $FF_1$ we present the projections of the most probable tokens for the intermediate inputs $x_i$. For $FF_2$ we show most improbable tokens for the VJPs $\delta_i$.
    The color indicates the norm of the gradient's $\delta_i$.
    In this example, the prompt is ``Lionel Messi plays for'', to which the model responds with ``Barcelona''. The new target is ``Paris''. 
    Notably, the target token ``Paris'' is evident through the majority of $FF_2$'s VJPs, which are the vectors that are injected into $FF_2$'s weights. On the other hand, $FF_1$ reflects not only the tokens that the gradients try to inject into $FF_1$, but also the model's intermediate predictions at each MLP layer during the forward pass.
    It is worth noting that the presence of tokens with a less clear meaning, such as ``VIDIA'' in $FF_2$ results from the projection of vectors with almost zero norm, as exampled in (\autoref{tab:messi FF2 samples}). This implies that these subcomponents of the gradient exert negligible influence when updating the model.}
    \label{fig:table_medium_messi}
\end{figure*}

\begin{figure*}[h]
    \centering
    \includegraphics[width=\textwidth]{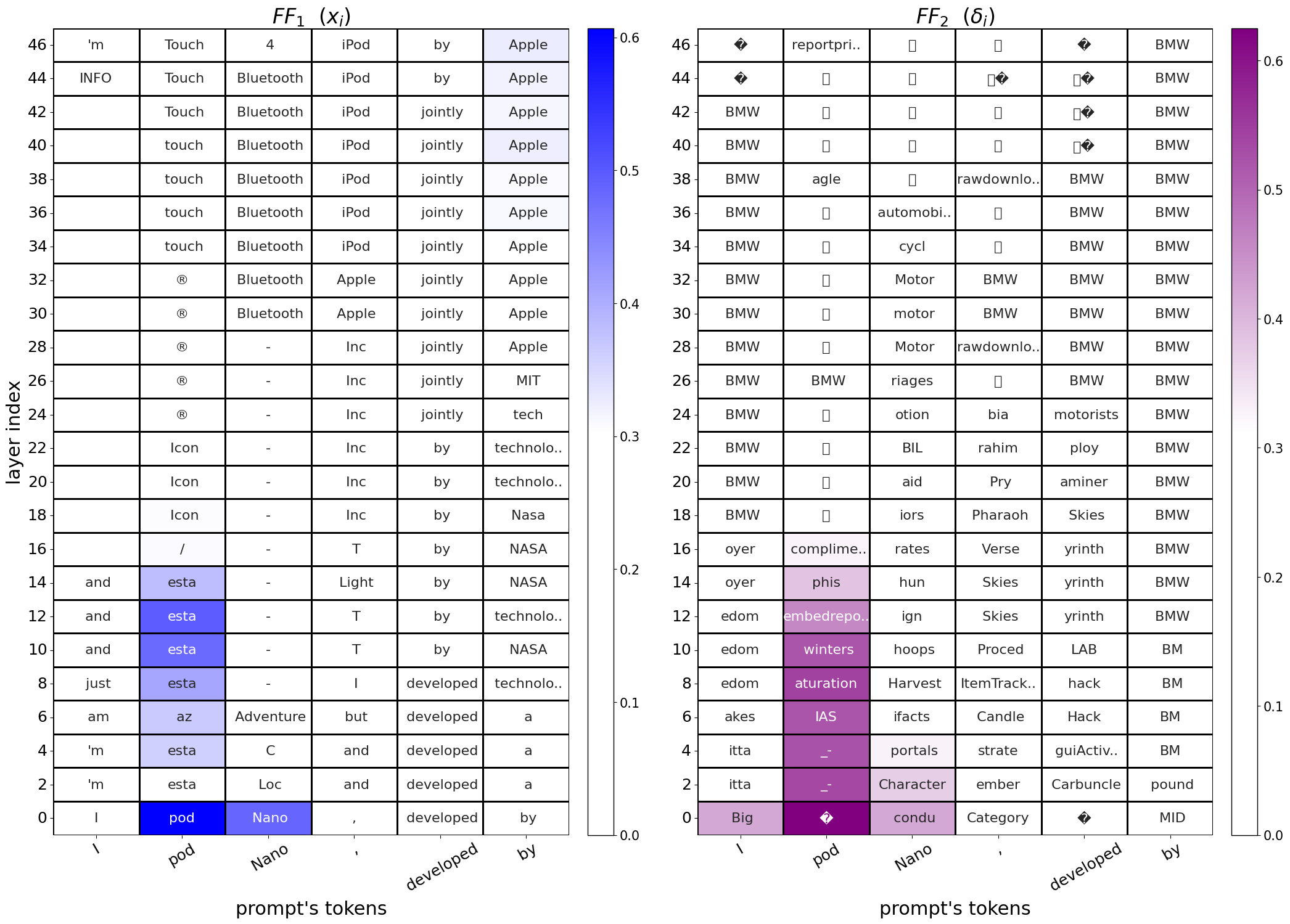}
    \caption{GPT2-xl MLP gradients via Logit Lens for ``Ipod Nano, developed by'' which we edit from ``Apple'' to ``BMW''.
    Due lack of space, we show only every second layer.
    The color scheme indicates that the primary focus of editing lies on the subject token ``Pod''. 
    We hypothesize that the word ``Ipod'' is highly related to ``Apple'', so by targeting ``pod'', the gradients edit the relation between Ipod and the company that created it.
    Question marks and empty boxes are non-English tokens.
    }
    \label{fig:table_xl_apple_bmw}
\end{figure*}

\begin{figure*}[h]
    \centering
    \includegraphics[width=\textwidth]
    {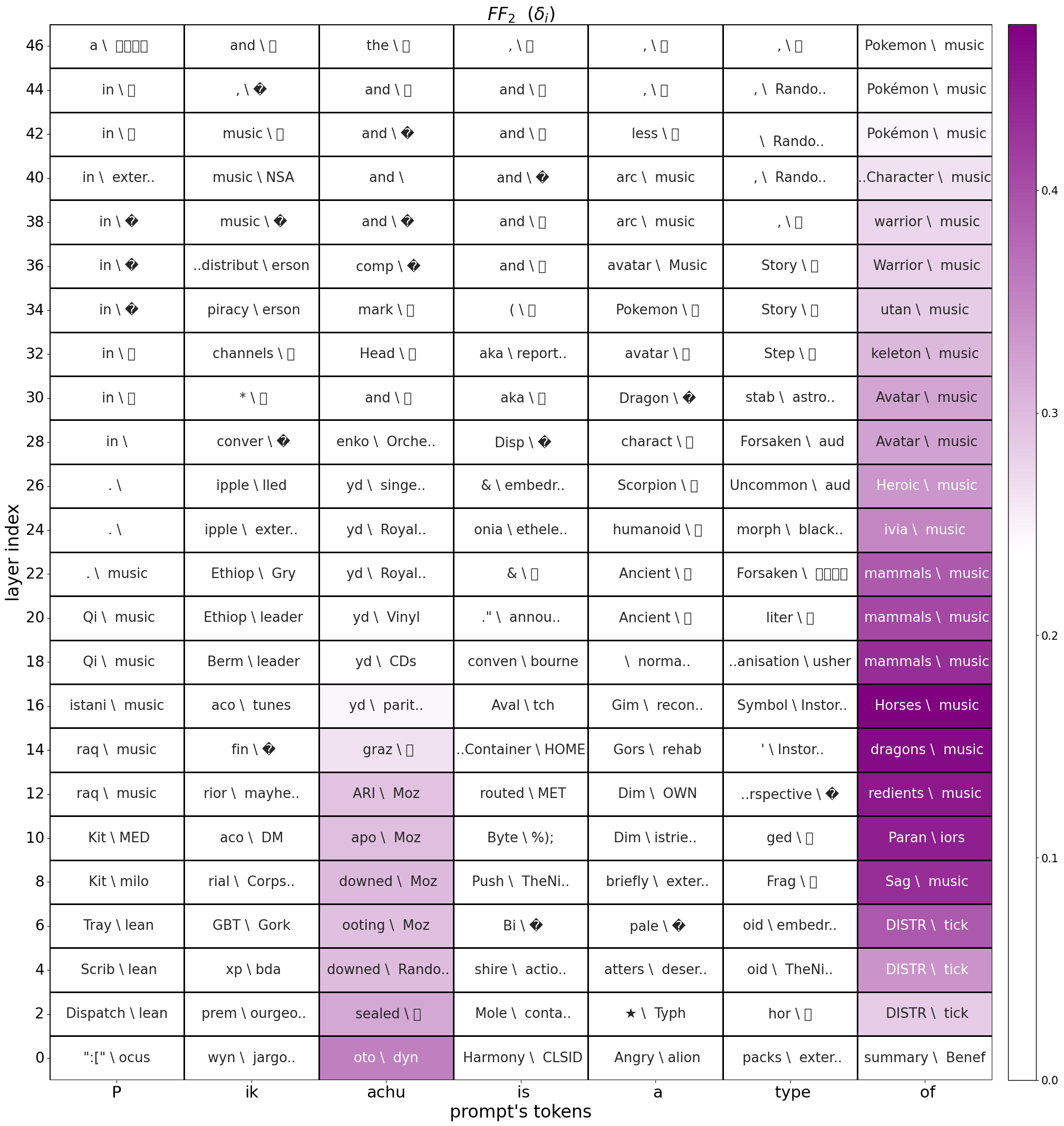}
    \caption{GPT2-xl $FF_2$ gradients via Logit Lens. The editing tries to change the prompt ``Pikachu is a type of'' to answer ``music'' instead of ``Pokémon''. 
    Every cell shows the most $<$ probable \textbackslash improbable $>$ tokens for its $\delta_i$.
    The fact that we see $FF_2$'s last token's project ``Pokémon / music'' means the edit tries to subtract from the model's weights the embedding of ``Pokémon'', while adding the embedding of ``music'' (same mechanism with ``mammals / music'').
    Regarding the presence of many non-English tokens (question marks and empty boxes), in \autoref{Normalize Logit Lens}, we present the same table with ``Normalized Logit Lens'', revealing the embedding of the target token ``music''.
    }
    \label{fig:table_xl_pickachu}
\end{figure*}

\begin{figure*}[h]
    \centering
    \includegraphics[width=\textwidth]
    {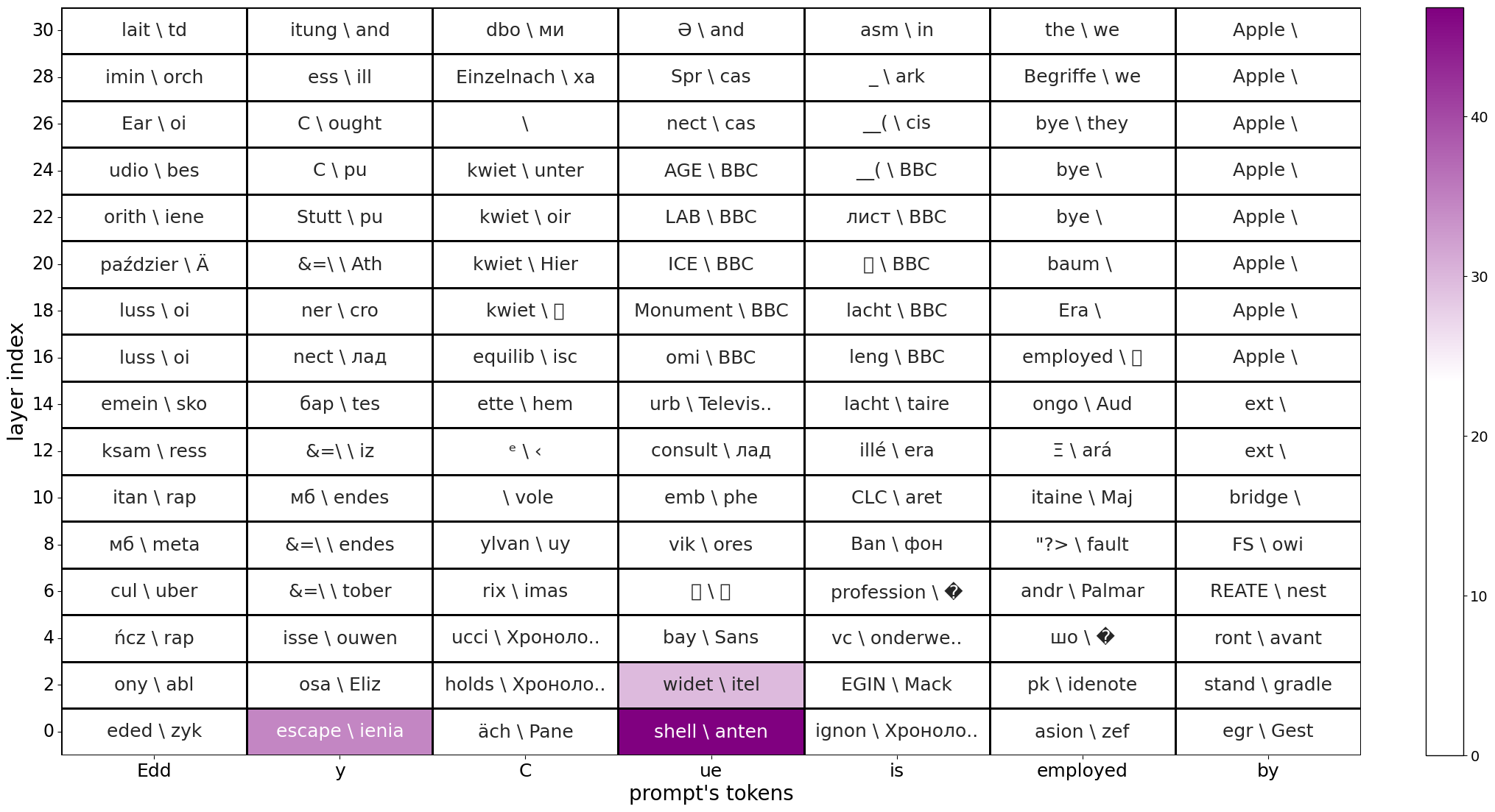}
    \caption{Llama2-7B $FF_2$ gradients via Logit Lens. The editing tries to change the prompt ``Eddy Cue is employed by'' to answer ``BBC'' instead of ``Apple'' (template was taken from the CounterFact dataset). 
    Every cell shows the most $<$ probable \textbackslash improbable $>$ tokens for its $\delta_i$.
    3 main patterns can be observed from the table: 
    (1) The part that has the largest VJPs by norm, which we assume is where most of the editing is done, is in the last subject token (``cu'') and during the first few layers of the model.
    (2) We do not see the target token ``BBC'' in the most improbable projection of the last token; instead, the most improbable one is the empty token ``''. When we examined these projections we found ``BBC'' to be only the third most improbable token. However, in other VJPs we notice this is indeed the most improbable projection.
    (3) The original answer of the model, ``Apple'', frequently emerges in the projections of the last token. As outlined in Section \ref{Trap and Shift}, given that we update the model with a negative learning rate, when the VJPs Logit Lens projection ranks a token as highly probable, subtracting the gradients from the model's weights equals to subtracting the embedding of this probable token from them .
    Consequently, this pattern illustrates the mechanism of ``shift'' within the context of ``imprint and shift'': We decrease the probability associated with ``Apple'' in pursuit of increasing the probability of ``BBC''.
    }
    \label{fig:table_llama_eddy_cua}
\end{figure*}

\begin{figure*}[h]
    \centering
    \includegraphics[width=\textwidth]
    {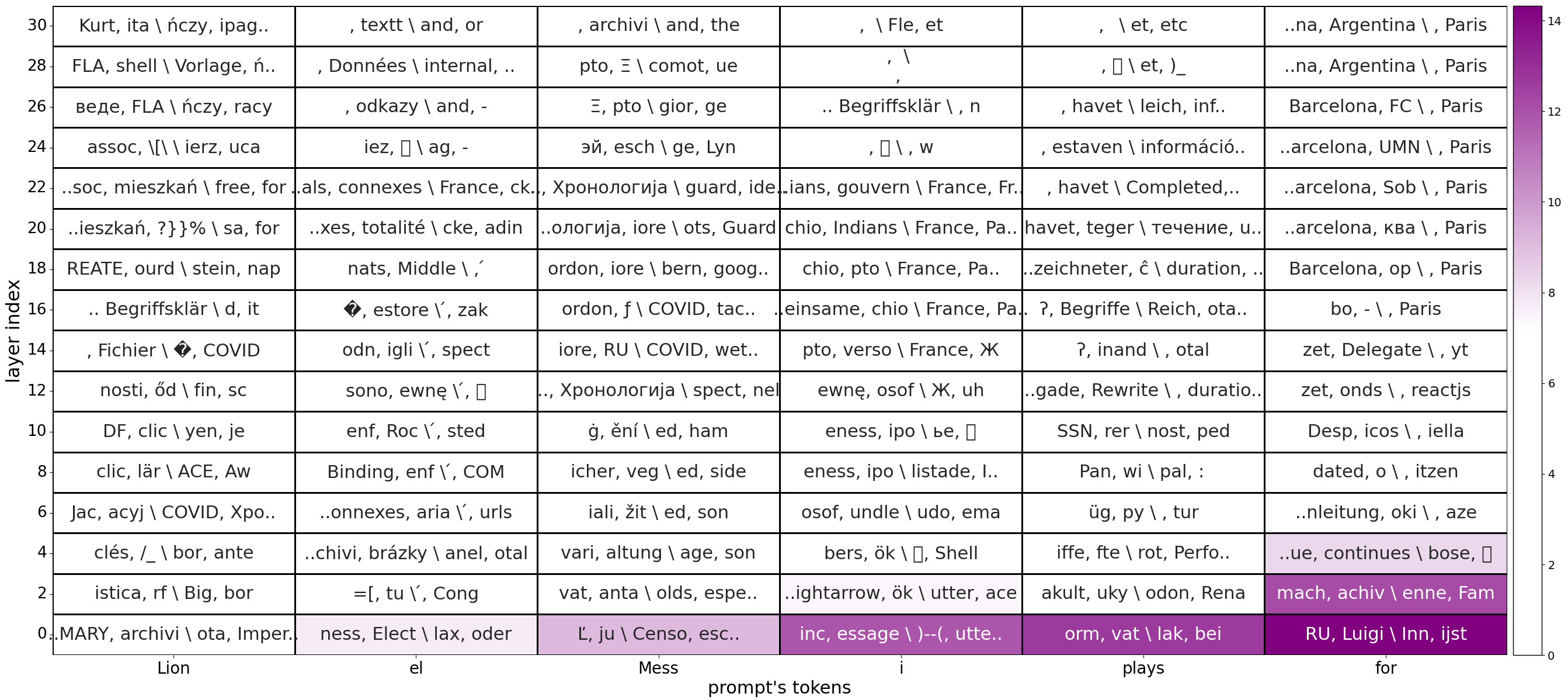}
    \caption{Llama2-7B $FF_2$ gradients via Logit Lens. The editing tries to change the prompt ``Lionel Messi plays for'' to answer ``Paris'' rather than ``Barcelona''. 
    Every cell shows the  \textbf{2 most probable and improbable tokens} of the VJPs $\delta_i$ in the format of $<$ probable \textbackslash improbable $>$.
    The target token, ``Paris'' is the second most improbable token in the projections of the last prompt's token (``for''), only second to the empty token (``'').
    The projection's most probable tokens of the last prompt's token are ``Barcelona'' (the final prediction of the model for this prompt), with some projections also revealing ``Argentina'' (which was the second most probable prediction of the model).
    In addition, the projections of the last subject token from the prompt (``i'') show that among the most improbable tokens there are appearances of ``France'', which can be associated to its capital, ``Paris''.
    Together, these projections illustrate how backpropogation editing tries to add into the model's $FF_2$'s neurons (weights) the embedding of ``Paris'' (and related concepts, e.g. ``France'') while removing the embeddings of ``Barcelona'' and ``Argentina''.
    }
    \label{fig:table_llama_messi}
\end{figure*}

\clearpage
\clearpage

\section{Additional Empirical Results}

\subsection{Impact of Different Segments of the Prompt in Every Layer}
\label{Editing Mass of Every Layer}
In \autoref{Editing Mass}, we elucidate how comparing the norm of each VJP ($\delta_i$) involved in the construction of a gradient matrix can reflect which layers are updated more than others. Similar comparison is conducted for the tokens from the edited prompt, uncovering that certain layers and segmentation from the prompt do not contribute significantly to the updating process.

We extend this analysis from \autoref{Editing Mass} to demonstrate the ability to identify the main editing matrices in every type of module in LMs \autoref{fig:tracing_full2}.

We will solely discuss the results related to the MLP layers, as we did not examine the attention modules in our study.
Both the $FF_1$ module (\textit{mlp.c\_fc}) and $FF_2$ module (\textit{mlp.c\_proj}) demonstrate that the primary editing occurs around the first quarter of layers by the last subject token, and around the middle of the layers by the last token. The majority of the other layers and tokens exhibit VJP norms close to zero, indicating that they scarcely contribute to updating the model's weights.
The same pattern is also evident for the VJPs between transformer layers (\textit{transformer.h}), as, if we disregard Dropouts and Layer Norms, the VJPs of $FF_2$ are the same as the one between the transformer layers.

We also provide similar figures where, instead of plotting the norm of the VJPs, we compare the norms of the intermediate inputs to every layer ($x_i$) (\autoref{fig:xi_norms}).
The inclusion of those figures is solely to emphasize that there is no correlation between the norm of $x_i$ and $\delta_i$.

\begin{figure*}[h!]
    \centering
    \includegraphics[width=\linewidth]{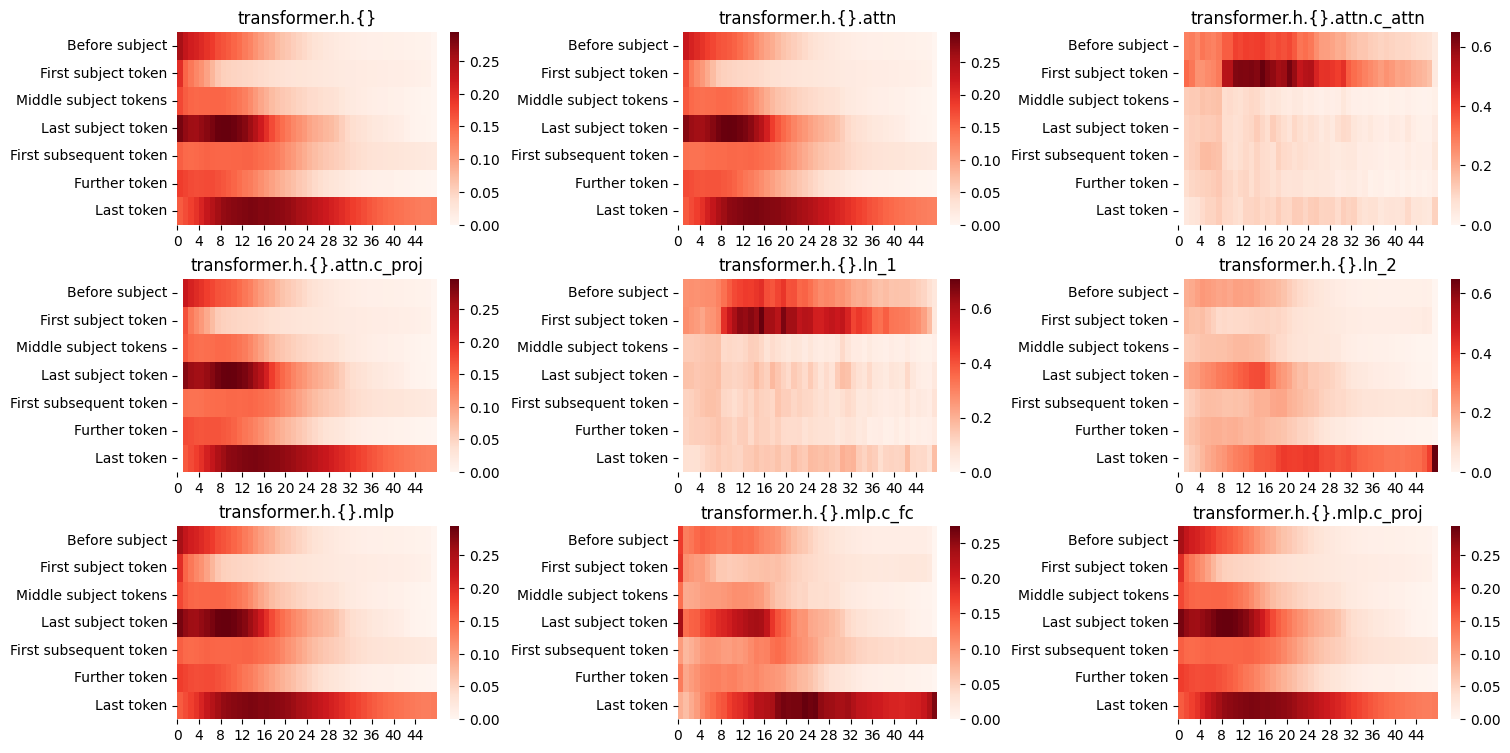}
    \caption{$\delta_i$ norm for each of GPT2-xl sub-module (names are according to the HuggingFace implementation). This is an extension to \autoref{Editing Mass}, showing how our approach is easily adapted to any submodule of the model.
    The layer of $FF_1$ is represented by \textit{mlp.c\_fc} and $FF_2$ by  \textit{mlp.c\_proj} .
    }
    \label{fig:tracing_full2}
\end{figure*}

\vspace{4em}
\begin{figure*}[h!]
    \centering
    \includegraphics[width=\linewidth]{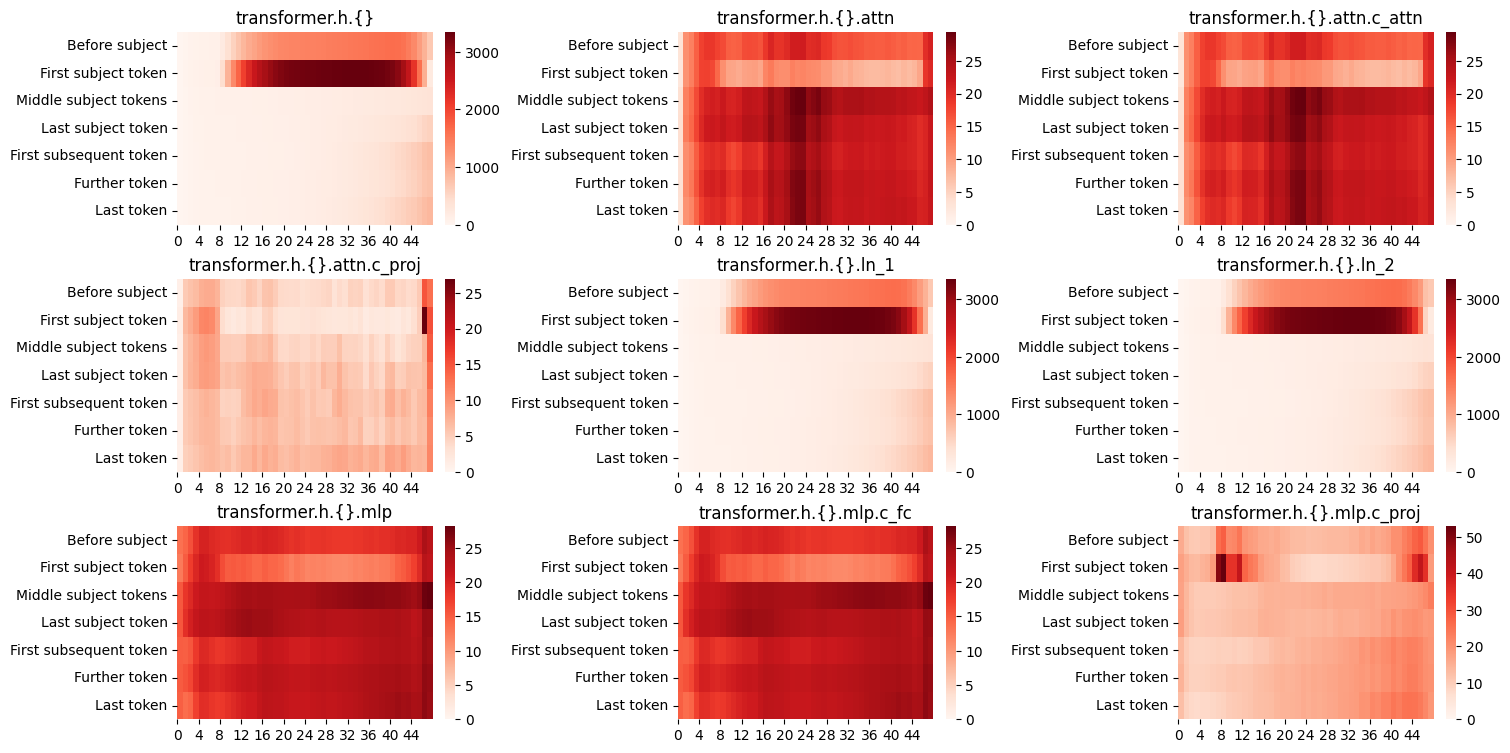}
    \caption{$x_i$ norm for GPT2-xl. Provided to show no correlation with \autoref{fig:tracing_full2}.}
    \label{fig:xi_norms}
\end{figure*}

\clearpage
\clearpage

\subsection{The Ranks of $FF_1$ and the Models' Original Answer}
\label{The Ranks of $FF_1$ and the Models' Original Answer}
In \autoref{What Tokens The Gradients Represent} we examine how $FF_2$'s VJPs rank the target tokens, the tokens we try to learn during backpropagation.
One possible interpretation of our analysis is to examine the backward pass according to the gradual change in the embeddings it tries to inject (add or subtract) into the model.
Previous works conduct similar analysis with the hidden states from the forward pass \citet{haviv2022understanding, geva-etal-2022-transformer}. Those works examined the gradual build in LMs' forward pass prediction, from the perspective of the last token in the edited prompts.
In this section, we expand upon these examinations by measuring the LL rank for the original answers outputted by the forward pass (before editing), and by observing these LL ranks from the perspective of $FF_1$'s spanning set. 

The presented results are based on 100 distinct edits using a single backward pass per edit. 
We employ GPT2 and Llama2-7B, and the edited prompts and targets were taken from the CounterFact dataset.

\paragraph{$FF_2$}: 
In \autoref{What Tokens The Gradients Represent} we presented the ranking of the target token for GPT2-xl. Subsequently, we present the rank of the last prompt's token VJP (annotated with $\delta_n$ in \autoref{observation: initial backward vjp}) also for GPT2-medium and Llama-7B.
In \autoref{fig:FF2_last_token_rank_for_target} we can see that all models' LL projections rank the target token as one of the most improbable tokens along most of the layers, with some degradation in the first few layers. We associate this drop to the gap of LL in projection vectors from earlier layers.

\begin{figure}[h!]
    \includegraphics[width=\columnwidth]{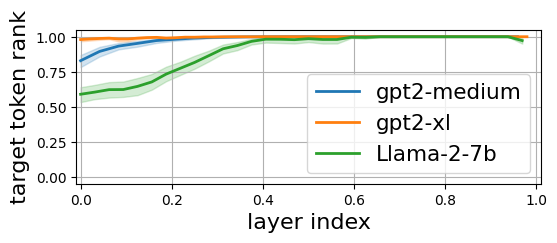}
    \caption{The Logit Lens rank of the target token in VJP of the last token from the edited prompt, $\delta_n$. In order to show different models in the same figure, we normalize the layer indices and the rank of the token in the vocabularies (which is 50K for GPT2 and 32K for Llama2). All models' $\delta_n$ assign the target token as one of the most improbable tokens along most of their layers.}
    \label{fig:FF2_last_token_rank_for_target}
\end{figure}

We repeat the same experiment and measure the rank of the original token predicted by the model. According to \autoref{formula: NLL direct VJP} and \autoref{observation: initial backward vjp}, the embedding of a token in the gradient should be discernible if it is the target token or if its 
probability in the model's final output is relatively high.

The results depicted in \autoref{fig:ranks ff2 for answer} illustrate that the LL rank of the final answer is relatively low in the model's last layers, although not at the lowest possible level.

In \autoref{The gradients of $FF_2$} and \autoref{Trap and Shift}, we delved into how the injection of the gradient vector with a high LL rank of the target token implies that the update aims to enhance the target probability in the output. Similarly, the observed pattern regarding the LL rank of the model's final answer suggests that the updates attempt to diminish the probabilities of the final answers in the model's output, but in a much smoother manner than those of the target token.


\begin{figure}[h!]
    \centering
    \includegraphics[width=\columnwidth]{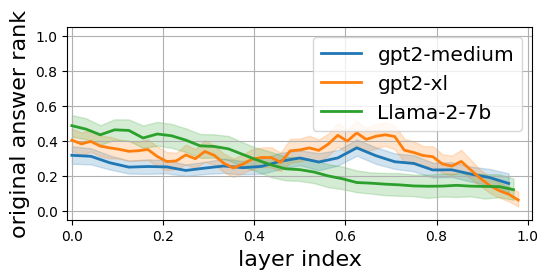}
    \caption{
    The Logit Lens ranks for the models' actual answers according to $FF_2$ gradients of the last prompts' token (layer indices and ranks are normalized). In the later layers, the rank of the original answers suggests that the model subtracts the embedding of those tokens from the model's weights.}
    \label{fig:ranks ff2 for answer}
\end{figure}

\begin{figure}[h!]
\centering
\subcaptionbox{Only correctly answered prompts}{\includegraphics[width=\columnwidth]{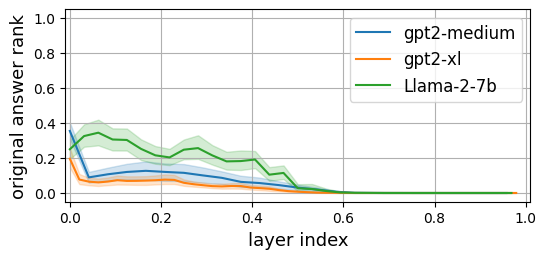}}
\subcaptionbox{Correctly and incorrectly answered prompts}{\includegraphics[width=\columnwidth]{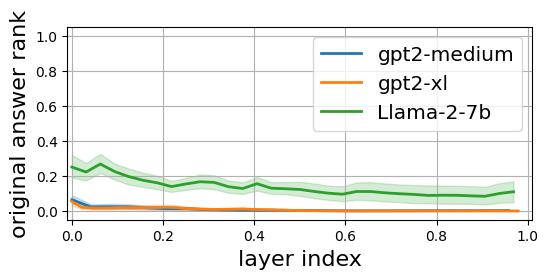}}
\caption{
$FF_1$'s $x_i$ ranks for the model's actual answers. The main difference between the two figures is that when the models answer incorrectly, they mostly answer with function words, such as ``a'' and ``the'', which seems to have a constant rank from the earlier layers, while answering the actual factual answers changes the answers' rank gradually in the first half of the layers.
The meaning of these graphs is that during fine-tuning, the embeddings injected into the models' weights are those of $x_i$.
}
\label{fig:ranks ff1 for answer}
\end{figure}

\begin{figure}[h!]
    \centering
    \includegraphics[width=\columnwidth]{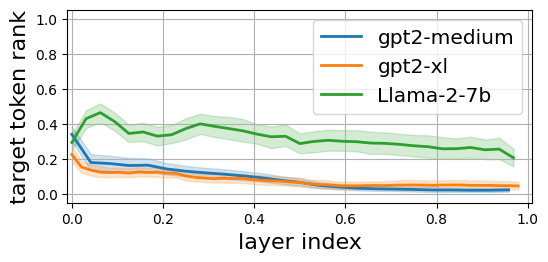}
    \caption{
    $FF_1$'s $x_i$ ranks for the edits' targets, reflecting the intermediate predictions the models give to them at each layer.}
    \label{fig:ranks ff1 for target}
\end{figure}

\paragraph{$FF_1$}: Since this layer's $\delta_i$ is not projectable via LL, we use its inputs $x_i$ as its spanning set. Analyzing the ranks of $x_i$ is almost identical to the analysis of \citet{haviv2022understanding}. We share our analysis in \autoref{fig:ranks ff1 for answer}, mainly to emphasize that gradients write into $FF_1$'s weights the inputs $x_i$ from the forward pass.
The gradual build in the models' predictions becomes more apparent when we filter out examples where the model answers CounterFact's prompts incorrectly, accounting for approximately $84\%$ of the instances with GPT2-medium/xl and Llama.\footnote{Most of the time, they predict tokens like ``a'', ``is'', and ``the'', rather than factual notions in the context of the prompts.}\footnote{Llama utilizes two matrices that employ $FF_1$ as part of its implementation of SwiGLU as its MLP activation function. Notably, the input $x_i$ remains consistent for both matrices.}
In \autoref{fig:ranks ff1 for answer}, we also included the same analysis with 50 correctly answered prompts. Throughout our study, this is the only result we found to be affected by the distinction whether the models answer the original prompt correctly or incorrectly.

Similarly, when we plot the ranks of the target token for each layer, we only observe the forward passes' intermediate probability for the target token (\autoref{fig:ranks ff1 for target}).

\vfill\eject

\section{Normalized Logit Lens}
\label{Normalize Logit Lens}
We acknowledge the sensitivity of the Logit Lens to low-norm vectors. With vectors with close to zero norms, the tokens LL project tokens that resemble the projection of the zero vector.
In \autoref{Editing Mass}, we discussed that certain VJPs, $\delta_i$, exhibit low norms. We hypothesize that the norms of the VJPs reflect which parts of speech and layers the backward pass tries to edit more. 
We were interested in observing which tokens could be projected from gradients if we isolate the influence of these low norms.
Our investigation reveals that normalizing the projected vectors before applying the Logit Lens can be an appropriate solution to the variance in VJPs' norms. We named this method ``Normalized Logit Lens''.

A place we consider using this normalization is in creating the tables of \autoref{Additional Table Examples}. In \autoref{fig:table_medium_messi_normalize} and \autoref{fig:table_xl_pickachu_normed} we share two examples that replicate the setup from \autoref{Additional Table Examples} except we are using Normalized Logit Lens.

We provide this section to highlight the pattern we already mentioned in \autoref{What Tokens The Gradients Represent}: that most of $FF_2$ gradients ranks the target token as one the most improbable token.
We also want to suggest further work to consider using Normalized Logit Lens to examine low-norm vectors.

\begin{figure*}[h!]
    \centering
    \includegraphics[width=\linewidth]{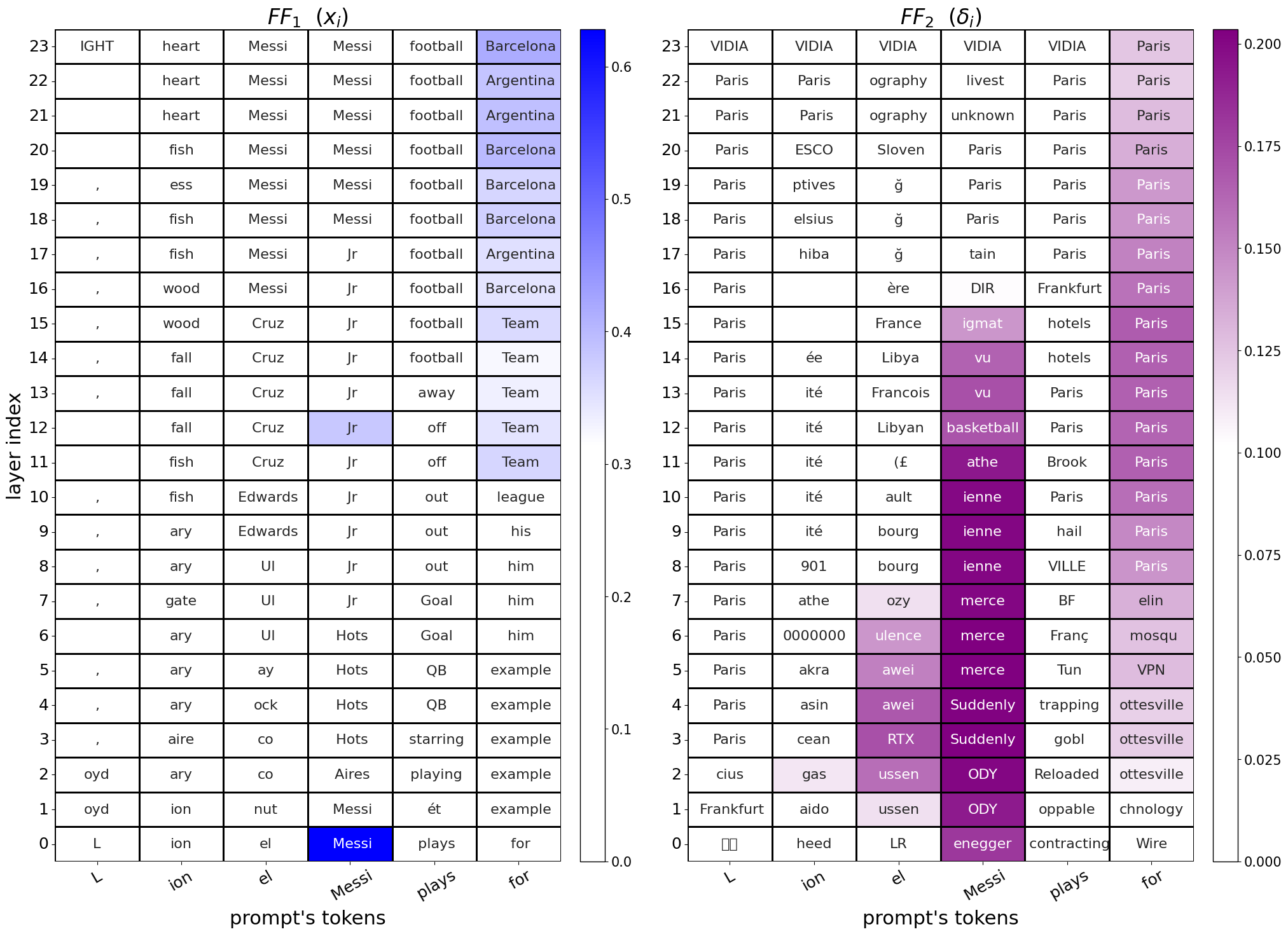}
    \caption{GPT2-medium MLP gradients projections with \textbf{Normalized Logit Lens}.
    Compared to the naive Logit Lens \autoref{fig:table_medium_messi}, we see more cells in $FF_2$ that project concepts similar to the target token ``Paris'' and less unclear tokens, such as ``VIDIA'' (notice the last layers or the column of the token``ion''). The coloring is according to the original norm of each representative vector (before normalization).}
    \label{fig:table_medium_messi_normalize}
\end{figure*}

\begin{figure*}[h]
\centering
\includegraphics[width=\linewidth]
{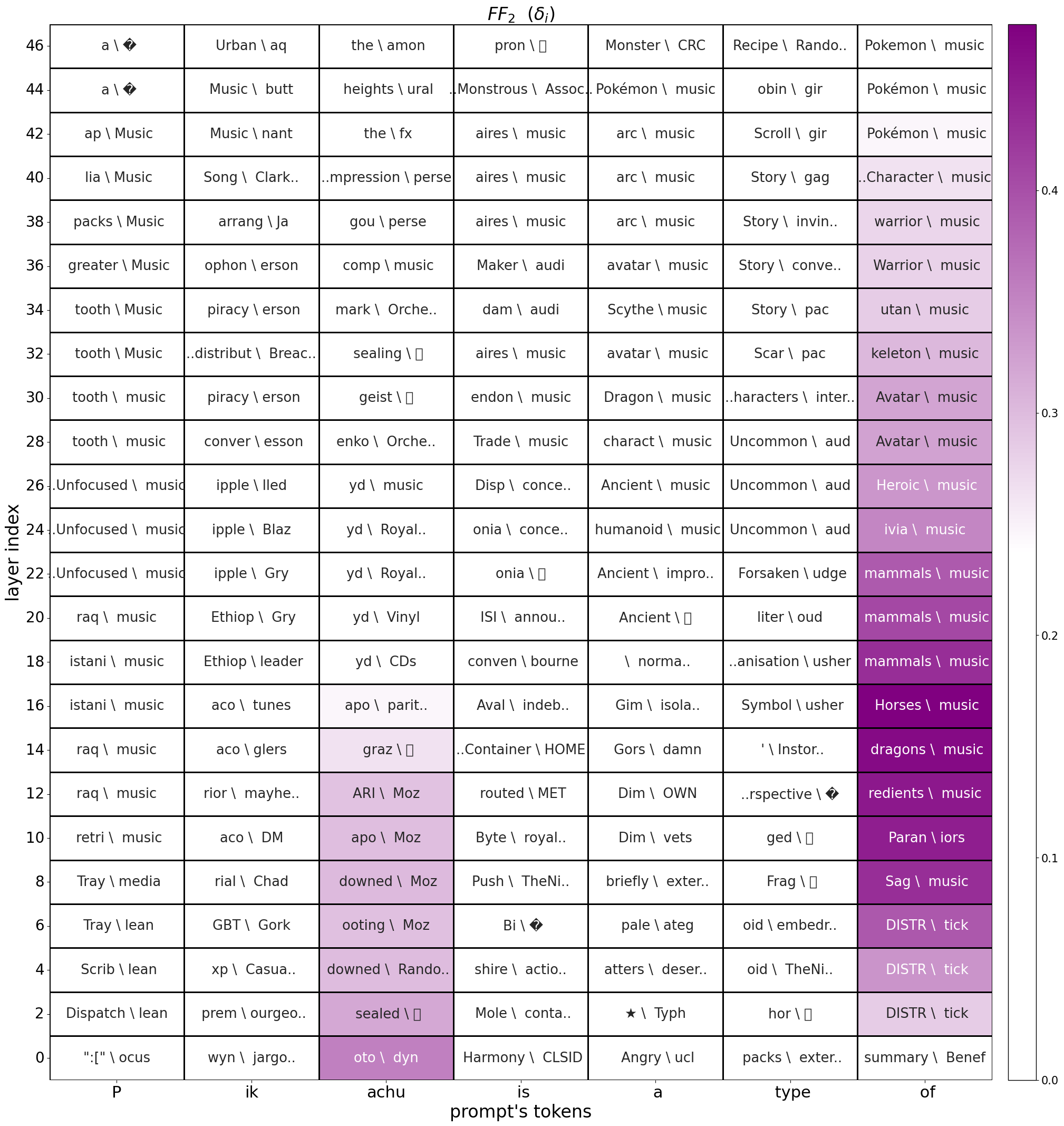}
\caption{GPT2-xl $FF_2$ gradients via our spanning set interpretation and \textbf{normalized Logit Lens}. 
Every cell shows the most $<$ probable \textbackslash improbable $>$ tokens for its $\delta_i$.
Compared to \autoref{fig:table_xl_pickachu}, we observe fewer non-English tokens in the projections and more tokens associated with the target token ``music''.}
\label{fig:table_xl_pickachu_normed}
\end{figure*}

\clearpage
\clearpage

\section{Editing Based on the ``Shift'' Mechanism}
\label{Manual Gradients Editing: appendix}
\label{appendix: editing}

\begin{table*}[ht!]
\vskip 0.15in
\begin{center}
\begin{sc}
\begin{tabular}{lccc}
\toprule
Method & Efficacy $\uparrow$ & paraphrase $\uparrow$ & neighborhood $\uparrow$  \\
\midrule
original model & 0.4$\pm$ 6.31  &  0.4$\pm$ 4.45  &  11.21$\pm$ 19.77
\\
Fine-tuning (mlp 0) & 96.37$\pm$ 18.7  &  7.46$\pm$ 19.44  &  10.12$\pm$ 18.02  
\\
Fine-tuning (mlp 35) & 100.0$\pm$ 0.0  &  46.1$\pm$ 40.25  &  5.59$\pm$ 13.19
\\
\midrule
MEND & 71.4$\pm$ 45.19  &  17.6$\pm$ 31.47  &  7.73$\pm$ 16.27 
\\
ROME &  99.4$\pm$ 7.72  &  71.9$\pm$ 37.22  &  10.91$\pm$ 19.53 
\\
MEMIT & 79.4$\pm$ 40.44  &  40.7$\pm$ 41.64  &  10.98$\pm$ 19.44 
\\
\textbf{Forward pass shift}  &  99.4$\pm$ 7.72  &  41.55$\pm$ 41.67  &  6.02$\pm$ 14.0
\\
\bottomrule
\end{tabular}
\end{sc}
\caption{Single editing results of GPT2-xl on CounterFact benchmark}
\label{tab: Manual grad appendix - editing results}
\end{center}
\vskip -0.1in
\end{table*}

\begin{table*}[ht!]
\vskip 0.15in
\begin{center}
\begin{sc}
\begin{tabular}{lcc}
\toprule
Method & N-gram entropy $\uparrow$ & WikiText Perplexity $\downarrow$ \\
\midrule
original model & 626.94$\pm$ 11.56  &  93.56$\pm$ 0.0
\\
Fine-tuning (mlp 0) & 618.81$\pm$ 52.74  &  199.84$\pm$ 1392.66  
\\
Fine-tuning (mlp 35) & 618.5$\pm$ 28.57  &  103.42$\pm$ 11.69
\\
\midrule
MEND & 623.94$\pm$ 23.14  &  94.96$\pm$ 0.66 
\\
ROME &  622.78$\pm$ 20.54  &  137.38$\pm$ 786.35 
\\
MEMIT & 627.18$\pm$ 12.29  &  93.6$\pm$ 0.11
\\
\textbf{Forward pass shift}  &  622.45$\pm$ 21.46  &  93.66$\pm$ 1.08  
\\
\bottomrule
\end{tabular}
\end{sc}
\caption{Examining the general ability of GPT2-xl to generate text after it has been applied with a single edit from CounterFact.}
\label{tab: Manual grad appendix - generate text abilities}
\end{center}
\vskip -0.1in
\end{table*}

In \autoref{Editing Directly with Target Embedding} we introduce ``forward pass shifting'', a method for LM knowledge editing through the approximation of the gradient matrix. In this section, we provide additional results and implementation details.
We want to remind that our intention is not to propose a new alternative editing method, but rather to explore the concept of injecting knowledge into the model in an approximate manner to how backpropagation operates.

\subsection{Comparison with Naive Backpropagation}
\label{Comparison with Naive Backpropagation}

We check our method's ability to perform a single edit at a time of a given prompt, such that after editing, the most probable answer will be a selected target token. 
Our baselines for comparison are performing the same single Backpropagation editing with SGD and Adam.
We use 50 samples from CounterFact \cite{meng2022locating} for the edited prompts and target.
For each editing method and sample, we measure the method's sensitivity to different learning rates by checking a range of possible values and finding the minimum one to achieve a successful edit.
We also monitor the model's general degradation after its update, using perplexity on maximum 256-tokens-long samples from WikiText \cite{merity2016pointer}. 
The findings are presented in \autoref{fig:Manual grad main paper figures}. Evidently, our approach's minimum learning rate shows less variation compared to SGD. Furthermore, its stability to editing, as indicated by perplexity, surpasses that of Adam and SGD in earlier layers and is identical in the latter ones.

\begin{figure}[t]
\centering
\subcaptionbox{Minimum learning rate for achieving a successful edit}{\includegraphics[width=1\columnwidth]{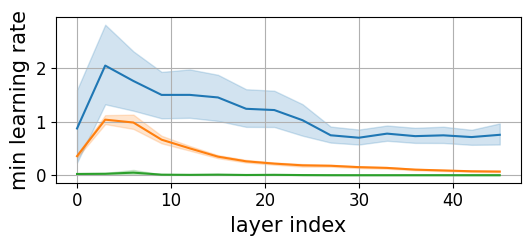}}
\subcaptionbox{Perplexity at minimum learning rate editing}{\includegraphics[width=1\columnwidth]{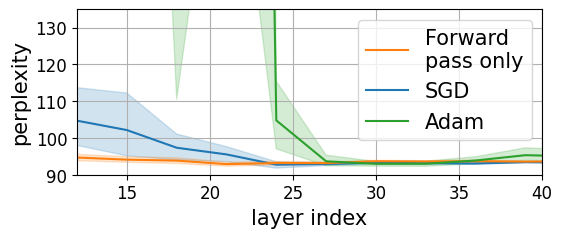}}
\caption{Editing only a single $FF_2$ matrix of GPT2-xl on CounterFact. By measuring a range of possible learning rate values we found the minimum that achieves a successful edit and measured the change in the model's general knowledge through perplexity on WikiText.}
\label{fig:Manual grad main paper figures}
\end{figure}

\subsection{Benchmark Implementation Details}
Based on the results from \autoref{Comparison with Naive Backpropagation}, we identified layers $30-40$ as the best potential editing layers. As for the learning rate, we examined values ranging from $0.14$ to $0.26$.
The presented results use layer 35 with a learning rate of $0.24$, which yielded the best results in the following benchmark.

The construction of the approximated gradient matrix is accomplished using the embedding of the target token.
This embedding is obtained from the decoding matrix. If the target token consists of more than one token according to the model's vocabulary, we select the prefix (the first token) to construct the target token.

Regarding the other editing methods and CounterFact benchmark implementation, we followed the implementations provided by \citet{meng2022mass} to create the results ourselves.
The post-editing metrics included in this benchmark, which we present, are as follows: (1) ``Efficacy'' - the accuracy (percentage of times) with which the model predicts the targeted token as the most probable one given the editing prompt. (2) ``Paraphrase'' - the accuracy of the model to answer paraphrases of the edited prompt (also known as ``Generalization''). (3) ``Neighborhood'' - the accuracy of the model on prompts from similar domains as the edited prompt, which we do not wish to change (also known as ``Specificity'').
(4) ``N-gram entropy'' measures the weighted average of bi- and tri-gram entropies, reflecting fluency level.
In addition, we included in the benchmark the perplexity score of the 100 first sentences from WikiText (version wikitext-2-raw-v1, test split\footnote{\url{https://huggingface.co/datasets/wikitext}}) that are at least 30 characters long. Moreover, each prompt was truncated to its first 256 tokens. The score of this metric reflects how much the model's original ability to generate text has changed.

The comparison was conducted on 1000 samples from CounterFact, and we benchmarked against state-of-the-art methods, such as ROME \cite{meng2022locating}, MEMIT \cite{meng2022mass}, and MEND \cite{mitchell2021fast}.

\subsection{Results}
In \autoref{tab: Manual grad appendix - editing results}, we present the editing results regarding the model's ability to modify internal knowledge. In \autoref{tab: Manual grad appendix - generate text abilities}, we analyze the effect of editing on the model's ability to generate generic text.

If we filter out editing methods that cause drastic degradation to the model's ability to generate text (where n-gram entropy falls below $620$), our method achieves the best editing results on the edited prompt (Accuracy), tying only with ROME. With regards to editing the paraphrased prompts, the ``forward pass shifting'' method achieves comparable results, but falls behind ROME.
However, one of our method's limitations lies in its performance with neighborhood prompts, where the edited model altered prompts we do not anticipate to be changed.

\subsection{Discussion}
``Forward pass shift'' demonstrates successful knowledge editing compared to methods that employ much more complex implementations. Additionally, our method shows minimal impact on the model's ability to generate text, addressing one of the main challenges of fine-tuning in precise knowledge editing.

From our understanding, ``forward pass shift'' is the simplest editing method in terms of algorithm complexity, as it only requires a single forward pass.
The low score of this method in terms of neighborhood editing may be attributed to mistakenly editing activation patterns of $FF_2$ that are also shared with similar prompts.
Future studies could explore the capability of ``forward pass shift'' to handle multiple edits at once, or to incorporate multiple iterations in its implementation (multiple forward passes).

\end{document}